\documentclass{article}

\usepackage{arxiv}
\usepackage[hyphens]{url}  
\usepackage[utf8]{inputenc} 
\usepackage[T1]{fontenc}    
\usepackage{hyperref}       
\usepackage{url}            
\usepackage{booktabs}       
\usepackage{amsfonts}       
\usepackage{nicefrac}       
\usepackage{microtype}      
\usepackage{lipsum}		
\usepackage{graphicx}
\usepackage{natbib}
\usepackage{doi}
\usepackage{times}  
\usepackage{helvet}  
\usepackage{courier}  

\usepackage{graphicx} 
\usepackage{natbib}  
\usepackage{caption} 
\DeclareCaptionStyle{ruled}{labelfont=normalfont,labelsep=colon,strut=off} 
\usepackage{amssymb}
\usepackage{amsthm}
\usepackage{stmaryrd}
\usepackage{mathtools}
\usepackage{comment}
\usepackage{amsmath}
\usepackage[symbol]{footmisc}
\usepackage{algorithm}
\usepackage{algorithmic}
\newtheorem{theorem}{Theorem}
\newtheorem{proposition}{Proposition} 
\newtheorem{observation}{Observation} 

\newtheorem{corollary}{Corollary}
\newtheorem{lemma}{Lemma}

\newcommand\ci{T}

\title{Saving Stochastic Bandits from  Poisoning Attacks  via Limited Data  Verification}

\author{ {\hspace{1mm} Anshuka Rangi}
\\
	University of California San Diego, USA\\
	\texttt{arangi@ucsd.edu} \\
	\And
{\hspace{1mm} Long Tran-Thanh} \\
	 University of Warwick, UK\\
	\texttt{long.tran-thanh@warwick.ac.uk} \\
\And
{\hspace{1mm}  Haifeng Xu} \\
	  University of Virginia, USA\\
	\texttt{hx4ad@virginia.edu} \\
	\And
{\hspace{1mm}   Massimo  Franceschetti} \\
	  University of California San Diego, USA\\
	\texttt{massimo@ece.ucsd.edu} \\
}

\renewcommand{\undertitle}{}

\hypersetup{
pdftitle={A template for the arxiv style},
pdfsubject={q-bio.NC, q-bio.QM},
pdfauthor={David S.~Hippocampus, Elias D.~Striatum},
pdfkeywords={First keyword, Second keyword, More},
}

\begin{document}
\maketitle

\begin{abstract}
This paper studies bandit algorithms under data poisoning attacks in a bounded reward setting. We consider a strong attacker model in which the attacker can observe both the selected actions and their corresponding rewards, and can contaminate the rewards with additive noise.
We show that \emph{any} bandit algorithm with regret $O(\log T)$ can be forced to suffer a   regret $\Omega(T)$ with an expected amount of contamination  $O(\log T)$. {This amount of contamination is  also necessary, as we prove that there exists an $O(\log T)$ regret bandit algorithm, specifically the classical Upper Confidence Bound (UCB), that requires $\Omega(\log T)$   amount of contamination  to suffer regret $\Omega(T)$.} To combat such poisoning attacks, our second main contribution is to propose verification based mechanisms, which use limited \emph{verification} to access a limited number of uncontaminated rewards.
In particular, for the case of unlimited verifications, we show that with $O(\log T)$ expected number of verifications, a simple modified version of the Explore-then-Commit type bandit algorithm can restore the order optimal $O(\log T)$ regret \emph{irrespective of the amount of contamination} used by the attacker. 
We also provide a UCB-like verification scheme, called Secure-UCB, that also enjoys full recovery from any attacks, also with $O(\log T)$ expected number of verifications.
To derive a matching lower bound on the number of verifications, we also prove that for any order-optimal bandit algorithm, this number of verifications $\Omega(\log T)$  is necessary to recover the order-optimal regret. On the other hand, when the number of verifications is bounded above by a budget $B$, we propose a novel algorithm, Secure-BARBAR, which provably achieves $\tilde{O}(\min\{C,T/\sqrt{B} \})$ regret with high probability against weak attackers (i.e., attackers who have to place the contamination \emph{before} seeing the actual pulls of the bandit algorithm), where $C$ is the total amount of contamination by the attacker, which breaks the known $\Omega(C)$ lower bound of the non-verified setting if $C$ is large.  
\end{abstract}


\section{Introduction}

Multi Armed Bandits (MAB) algorithms are often used   in web services \citep{agarwal2016making, li2010contextual},   sensor networks  \citep{tran2012long},   medical trials \citep{badanidiyuru2018bandits,rangi2019unifying}, and   crowdsourcing systems \citep{rangi2018multi}. The distributed nature of these applications makes these algorithms prone to third party attacks. For example, in web services  decision making critically depends on   reward collection, and this is prone to attacks that can impact observations and monitoring, delay or temper rewards, produce link failures, and generally modify or delete information through hijacking of communication links   \citep{agarwal2016making} \citep{cardenas2008secure,rangi2021learning}. { Making} these systems secure requires an understanding of the regime where the systems can be attacked, as well as designing ways to mitigate these attacks. In this paper, we study both of these aspects   in a stochastic MAB setting. 


We consider a data poisoning attack, also referred as  man in the middle (MITM) attack. In this attack, there are three agents: the environment, the learner (MAB algorithm), and the attacker. At each discrete time-step $t$, the learner selects an action $i_t$ among   $K$ choices,
the environment then generates a reward $r_t(i_t)\in[0,1]$ corresponding to the selected action, and attempts to communicate it to the learner. However, an adversary intercepts $r_t(i_t)$ and can contaminate it by adding noise {$\epsilon_t(i_t)\in [-r_t(i_t),1-r_t(i_t)]$}. It follows that the learner  observes the contaminated reward $r^o_t(i_t)=r_t(i_t)+\epsilon_t(i_t)$, and $r^o_t(i_t)\in [0,1]$. 
Hence, 
the adversary acts as a ``man in the middle'' between the learner and the environment. 
{We present  an upper bound on
both the 
amount of contamination, which is the total amount of additive noise injected by the attacker, and the number of attacks, which is the number of times the adversary contaminates the observations, sufficient  
to ensure that the regret of the algorithm is  $\Omega(T)$, where $T$ is the total time of interaction between the learner and the environment.} Additionally, we establish that this upper bound is order-optimal by providing a lower bound on the number of attacks and the amount of contamination. 


A  typical way to protect  a distributed system from  a MITM attack is to employ a secure channel between the  learner and the environment \citep{asokan2003man,sieka2007establishing,callegati2009man}. These secure channels ensure the  CIA triad: confidentiality, integrity, and availability  \citep{ghadeer2018cybersecurity,doddapaneni2017secure,goyal2019security}. 
Various ways to establish these   channels have been explored in the literature \citep{asokan2003man,sieka2007establishing, haselsteiner2006security, callegati2009man}.
 An alternative way to provide security is by auditing, namely perform data verification \citep{karlof2003secure}.  The idea of data verification or using trusted information 
 is also embraced in the learning  literature where small number of observations are verified \citep{charikar2017learning,bishop2020optimal}. Establishing a secure channel  or an effective auditing method or getting  trusted information is generally costly \citep{sieka2007establishing}. Hence, it is crucial to design algorithms that achieve security, namely the performance of the algorithm is unaltered (or minimally altered) in presence of attack, while limiting the usage of these additional resources. 

Motivated by these observations, we consider a \emph{reward verification} model in which the learner can access   verified (i.e. uncontaminated) rewards from the environment. This verified access   can be implemented through a secure channel between the learner and the environment, or using auditing. 
At any round $t$, the learner can decide whether to access the possibly contaminated reward $r^o_t(i_t)=r_t(i_t)+\epsilon_t(i_t)$, or to access the verified reward $r^o_t(i_t)=r_t(i_t)$. 
Since verification is costly, the learner faces a tradeoff between its performance in terms of regret, and the number of times access to a verified reward occurs. Second, the learner needs to decide when to access a verified reward during the learning process. We design an order-optimal bandit algorithm which strategically plans the verification, and makes no assumptions on the attacker’s strategy.  

Against this background, we make the following contributions in this paper:

\begin{itemize}
    \item First, in Section~\ref{sec: characterization of attacks} we 
    provide a tight characterisation about the total (expected) number of contaminations needed for a successful attack. Specifically, while it is well-known that with $O(\log{T})$ expected number of contaminations, a strong attacker can successfully attack \emph{any} bandit algorithm (see Section~\ref{subsec: upper bound successful attack} for a more detailed discussion), it is not known to date whether this amount of contamination is necessary. We fill this gap by providing a matching lower bound on the amount of contamination (Theorem 1). This result is based on a novel insight of UCB's behaviour, which may be of independent interest. Specifically, we show that for arbitrary (even adversarial) reward sequences, UCB will pull every arm at least $\log(T/2)$ times for sufficiently large $T$. Such  conversativeness property of UCB guarantees its robustness against any attack strategy with $o(\log T)$ contaminations.  Note that we also extend the state-of-the-art results on the sufficient condition by proposing a  simpler yet optimal attack scheme, which is oblivious to the bandit algorithm's actual behaviour (Proposition 1).
    
    \item We then consider bandit algorithms with verification as a means of defense against these attacks. 
    In our first set of investigations, we consider the case of having unlimited number of verification (Section~\ref{subsec: unlimited verification}). We first show that the minimum number of verification needed to recover from any strong attack is $\Theta(\log{T})$ (Theorem 2 and Corollary 2). We then propose an Explore-Then-Commit (ETC) based method, called Secure-ETC that can achieve full recovery from any attacks with this optimal amount of verification (Observation 1). 
    While Secure-ETC is simple, it might not stop the exploration phase before exceeding the time horizon. To avoid this situation, we also propose a UCB-like method called Secure-UCB, which also enjoys full recovery under optimal verification scheme (Theorem 3). 
    
    \item Finally, we consider the case when the number of verifications is bounded above by a budget $B$. We first show that if the attacker has unlimited contamination budget, it is impossible to fully recover from the attack if the verification budget $B = o(T)$ (Theorem 4). However, when the attacker also has a finite contamination budget $C$, as typically assumed in the literature, we  propose Secure-BARBAR, which achieves $\tilde{O}\bigg(\min\Big\{C, { T\log{({2}/{\beta})}}/{\sqrt{B}}\Big\}\bigg)$ regret against a weaker attacker (who has to place the contamination before seeing the actual pull of the bandit algorithm). It remains an intriguing open question whether there exists efficient but limited verification schemes against stronger attackers.   
\end{itemize}


\section{Preliminaries and Problem Statement}
\label{sec: problem description}

\subsection{Poisoning Attacks on Stochastic Bandits}
We consider the classical stochastic bandit setting under data poisoning attacks. In this setting, a learner can choose  from a set of $K$ actions for $T$ rounds.  At each round $t$, the learner chooses an action $i_t\in [K]$, triggers a reward $r_{t}(i_t)\in [0,1]$ and observes a possibly corrupted (and thus altered) reward $r^o_{t}(i_t)\in [0,1]$  corresponding to the chosen action. The reward $r_t(i)$ of action $i$  is sampled independently from a fixed unknown   distribution of action $i$. Let $\mu_i$ denote the expected reward of action $i$ and $i^*=\mbox{argmax}_{i\in[K]}\mu_i$.\footnote{For convenience, we assume $i^*$ is unique though all our conclusions hold when there are multiple optimal actions.}  Also, let $\Delta(i)=\mu_{i^*}-\mu_i$ denote the difference between the expected reward of actions $i^*$ and  $i$.   
Finally, we assume that $\{\mu_i\}_{i\in[K]}$ are unknown to both the \emph{learner} and the \emph{attacker}. 

The  reward $r^o_{t}(i_t)$ observed by the learner and the  true reward $r_t(i_t)$ satisfy the following relation
\begin{equation}
    r^o_{t}(i_t)=r_t(i_t)+\epsilon_t(i_t),
\end{equation}
where  the contamination $\epsilon_t(i_t)$ added by the attacker can be a function of $\{i_n\}_{n=1}^t$ and $\{r_n(i_n)\}_{n=1}^t$. Additionally, since $r^o_{t}(i_t)\in [0,1]$, we have that   $\epsilon_t(i_t)\in [-r_t(i_t),1-r_t(i_t)]$. 
If $\epsilon_t(i_t) \neq 0$, then the round $t$ is said to be \emph{under attack}. Hence, the {\emph{number of attacks}}    is  $\sum_{t=1}^T\mathbf{1}(\epsilon_t(i_t)\neq 0)$ and the \emph{amount of contamination}   is $\sum_{t=1}^T |\epsilon_t(i_t)|$. 

The regret $R^{\mathcal{A}}(T)$ of a learning algorithm $\mathcal{A}$ is the difference between the total expected true reward from the best fixed action and the total expected \emph{true} reward   over $T$ rounds, namely
\begin{equation}\label{eq:RegretOfAlg}
    R^{\mathcal{A}}(T)=T\mu_{i^*}-\mathbb{E}[\sum_{t=1}^T r_{t}(i_t)],
\end{equation}
The objective of the learner is to minimize the regret $R^{\mathcal{A}}(T)$. In contrast, the objective of the attacker is to increase the regret to at least $\Omega(T)$. As a convention, we  say the attack is ``successful'' only when it leads to $\Omega(T)$ regret \citep{jun2018adversarial,liu2019data}.   The first  question we address is the following. 

\noindent {\bf  Question 1: } {\it Is there a \emph{tight characterization} of the amount of contamination and the number of attacks leading to a regret of~~$\Omega(T)$ in stochastic bandits?}


\subsection{Remedy via  Limited  Reward Verification}

It is well known that no stochastic bandit algorithm can be resilient to data poisoning attacks {if the attacker has sufficiently large amount of contamination} \citep{liu2019data}. Therefore, to guarantee sub-linear regret {when the attacker has an unbounded amount of contamination} it is  necessary for the bandit algorithm  to exploit additional (and possibly costly) resources. We consider  one of the most natural resource --- \emph{verified rewards}. Namely, we assume that at any round $t$, the learner can  choose to access the true, uncontaminated reward of the selected action $i_t$,  namely, when  \emph{  round $t$ is verified} we have $r^o_t(i_t)=r_t(i_t)$. This process of accessing true  rewards is referred to as \emph{verification}.  If  the learner performs verification at each round, then {it is clear that} 
the regret of any bandit algorithm is unaltered in the presence of attacker. 
Unfortunately, this is  unrealistic because verification  is costly in practice. Therefore, the learner  has to carefully balance the regret and the number of verifications. 
This naturally leads to the second question that we aim to answer in this paper:

\noindent {\bf Question 2: } {\it 
Is there a \emph{tight characterization} of the number of verifications  needed by the learner to   guarantee the optimal $O (\log T)$ regret for \emph{any} poisoning attack?
}

 {Finally, we consider the case of limited amount of contamination from the attacker and limited number of verifications from the bandit algorithm. In the direction of studying this trade-off between contamination and verification, the third question that we aim to answer in this paper is:

\noindent {\bf Question 3: } {\it 
Can we improve upon the $\Omega{(C)}$ regret lowerbound if the attacker's contamination budget is at most $C$, and the number of verifications that can be used by a bandit algorithm is also bounded above by a budget $B$. 
}
}

In this paper we answer the three questions above.

\section{Tight Characterization for the Cost of Poisoning Attack} 
\label{sec: characterization of attacks}


In this section we  show that if an attack can successfully induce $\Omega(T)$ linear regret for any bandit algorithm,   both its expected number of attacks   and the expected amount of contamination must be $\Theta(\log T)$. In other words, 
 there exists a ``robust'' stochastic bandit algorithm that cannot be successfully attacked by any attacker with only $o(\log T )$ expected amount of contamination,
 and we show the celebrated UCB algorithm satisfies this property. 
 The key technical challenge in proving the above result is to show the sublinear regret of UCB against \emph{arbitrary} poisoning attack using at most  $o(\log T)$ amount of contamination. 
 In order to prove this strong result, we discover a   novel ``convervativeness'' property of the UCB algorithm which may be of independent interest and has already found application  in   completely different tasks \cite{Shi2021Neurips}. To complement and also to match the above  lower bounds of any successful attack,  we  design 
a data poisoning attack that can indeed use   $O(\log T)$ expected number of attacks to induce $\Omega(T)$ regret for any  order-optimal  bandit algorithm, namely any algorithm which has $O(\log T)$-regret in the absence of attack. Since $r_t^o(i_t)\in [0,1]$, this   implies that the attack would require at most $O(\log T)$ expected amount of contamination.





\subsection{Lower Bound on the  Contaminations} 
We    show that there exists an order-optimal bandit algorithm --- in fact, the   classical UCB algorithm --- which cannot be attacked with  $o(\log T)$ amount of contamination
 by \emph{any} poisoning attack strategy. This implies that if an attacking strategy is required to be successful for all order-optimal bandit algorithms, then the amount of contamination needed is at least $\Omega(\log T)$.
 Since the amount of contamination is bounded above by the number of attacks, this also implies that any attacker requires at least $\Omega(\log T)$ number of attacks to be successful. While adversarial attacks to bandits have been extensively studied recently, to our knowledge such a lower bound on the attack strategy is novel and not known before; previous results have mostly studied the upper bound, i.e, how much contaminations are need for successful attacks \cite{jun2018adversarial,liu2019data}.  

Here we briefly describe the well-known UCB algorithm \citep{auer2002finite}, and defer its details to Algorithm \ref{alg:UCB} in   Appendix \ref{append:AlgUCB}. 
At each round $t\leq K$, UCB selects an action in round robin manner. At each round $t>K$, the selected action $i_t$ has the maximum \emph{upper confidence bound}, namely
\begin{equation}\label{eq:ucb-def}
             i_t= \mbox{argmax}_{i\in[K]} \bigg( \hat{\mu}_{t-1}(i)+ \sqrt{\frac{8\log t}{N_{t-1}(i)}} \bigg),
\end{equation}
where $N_t(i)=\sum_{n=1}^t\mathbf{1}(i_n=i)$ is the number of rounds action $i$ is selected until (and including) round $t$, and 
\begin{equation}
        \hat{\mu}_{t}(i)=\frac{ \sum_{n=1}^t  r^o_{n}(i_n) \mathbf{1}(i_n = i)}{N_{t}(i)},
    \end{equation}
is the  empirical mean of action $i$ until round $t$. 
Note that the algorithm uses the \emph{observed} rewards.

The following Theorem \ref{thm:lowerBoundonUCB} establishes that the UCB algorithm will have sublinear regret $o(T)$  under any poisoning attack if the amount of contamination is $o(\log T)$. 
The proof of Theorem \ref{thm:lowerBoundonUCB} crucially hinges on the following ``conservativeness'' property about the UCB algorithm,  which  may be of independent interest.\footnote{Indeed,  Lemma \ref{lemma:Min number of pulls} has been applied in \citep{Shi2021Neurips} to the task of incentivized exploration in order to show that a \emph{principal}  can get sufficient feedback from every arm even if the \emph{agent} who pulls arms has completely different preferences from the principal.}     
\begin{lemma}[Conservativeness of UCB]
\label{lemma:Min number of pulls}
Let $t_0$ be such that ${t_0}/{(\log (t_0))^2} \geq 36K^2$. Then  for all $ t \geq t_0$ and any sequence of rewards $\{r^o_n(i)\}_{i\in [K],n\leq t}$ in $[0,1]$ (can even be adversarial), UCB will select every action at least $ \log (t/2)$ times up until round $t$. 
\end{lemma} 
 Lemma \ref{lemma:Min number of pulls} is inherently due to the design of the UCB algorithm. { Its proof does \emph{not} rely on
 the rewards being stochastic,
 and it  holds deterministically ---  i.e., at any time $t \geq t_0$, UCB will pull each action at least $ \log (t/2)$ times. 
 }
This lemma leads to the following theorem.

 \begin{theorem}\label{thm:lowerBoundonUCB}
 For all $0<\epsilon<1$ and $\alpha>0$ such that $0<\epsilon\alpha\leq 1/2$, and for all $T  > \max\{(t_0)^{\frac{1}{1-\alpha \epsilon}}, \exp{(4^\alpha)}\}$, if the total \emph{amount} of contamination  by the attacker is  
    $\sum_{n=1}^T |\epsilon_n(i_n)|\leq {(\log T)^{1-\epsilon}}$,
then there exists a constant $c_1$ such that the expected regret of UCB algorithm is 
\begin{equation}
    R^{UCB}(T)\leq   c_1\big( T^{1-\alpha \epsilon} \max_i\Delta(i)+ \sum_{i \not = i^*}\log T/\Delta(i)\big),
\end{equation} 
which implies the regret $R^{UCB}(T)$ is $o(T)$.
 \end{theorem}
The constant  $\alpha$ in Theorem \ref{thm:lowerBoundonUCB} is an adjustable \emph{parameter} to control the tradeoff between the scale of time horizon $T$ ($T \geq \max\{(t_0)^{\frac{1}{1-\alpha \epsilon}}, \exp{(4^\alpha)}\}$) and the dominating term $(T^{1-\alpha \epsilon} \max_i\Delta(i))$   in the regret. If $\epsilon$ is small, then the larger $\alpha$ leads to a smaller regret, however  $T$ should be sufficiently large in order for us to see such a regret. 

The upper bound on the expected regret in Theorem \ref{thm:lowerBoundonUCB} holds if the total {amount} of contamination 
is at most $(\log T)^{1-\epsilon}$. 
Furthermore, if 
the total number of attacks is at most $(\log T)^{1-\epsilon}$, then using  $|\epsilon_t(i_t)|\leq 1$, we have that  $\sum_{n=1}^T |\epsilon_n(i_n)|\leq {(\log T)^{1-\epsilon}}$. 
Hence,  Theorem \ref{thm:lowerBoundonUCB} also  establishes that if the total number of attacks  is $o(\log T)$, then the expected regret of UCB is $o( T)$. Thus, the attacker requires at least $\Omega(\log T)$ amount  of contamination (or number of attacks) to ensure its success. 
 

 
The lower bound on the amount of contamination  in Theorem~\ref{thm:lowerBoundonUCB} cannot be directly compared with the upper bound in Proposition~\ref{thm:constantAttack} since the former assumes that the amount of contamination is bounded above by $o(\log{T})$ \emph{almost surely}, while the latter is a bound on the \emph{expected} amount of contamination. Instead, we consider the following corollary, which can be easily derived from Theorem~\ref{thm:lowerBoundonUCB} using Markov's inequality, and establishes the lower bound on the expected amount of contamination necessary for a successful attack.
\begin{corollary}
\label{corr:PAC lower bound of attacker for UCB}
For all $\epsilon\in (0,1)$ and  $T$ such that the conditions in Theorem~\ref{thm:lowerBoundonUCB} are satisfied,
if the expected amount of contamination by the attacker is at most $(\log{T})^{1-\epsilon}$, in other words $o(\log T)$, then the regret of UCB is $o(T)$. 
\end{corollary}




\subsection{Matching Upper Bound on    Contamination}
\label{subsec: upper bound successful attack}
We now show that there indeed exists attacks that can succeed with $O(\log T)$ attacks. Consider an attacker who tries to ensure  any action $i_A\in [K]$  to be selected by the bandit algorithm  at least $\Omega (T)$ times in expectation.
This implies that the expected regret of the bandit algorithm is  $\Omega(T)$ if $i_A\neq i^*$. We consider the following  simple attack, that pulls the observed reward down to $0$  whenever the target suboptimal action $i_A$ is not selected. Namely, 
\begin{equation}\label{eq:attackStrategy1}
    r_t^o(i_t)=\begin{cases} r_t(i_t)&\mbox{ if } i_t=i_A,\\
    0 &\mbox{ if } i_t\neq i_A.
    \end{cases}
\end{equation}
Equivalently,  the attacker adds $\epsilon_t(i_t)=-r_t(i_t)\mathbf{1}(i_t\neq i_A)$ to   the true reward $r_t(i_t)$. {Unlike the attacks in \cite{jun2018adversarial,liu2019data}, the attack in \eqref{eq:attackStrategy1} is oblivious to rewards, since it overwrites all the rewards observation by zero.}
The following proposition establishes an upper bound on the expected number of attacks sufficient to be successful.    
  
\begin{proposition}\label{thm:constantAttack}
For any stochastic bandit algorithm $\mathcal{A}$  with expected regret  
in the \emph{absence} of   attack   given by
\begin{equation}\label{eq:algo1}
    R^{\mathcal{A}}(T)=O\bigg(\sum_{i\neq i^*}\frac{\log^\alpha(T)}{(\Delta(i))^\beta}\bigg), 
\end{equation}
where $\alpha\geq 1$ and $\beta\geq 1$; and  for any target action $i_A\in [K]$; if an attacker follows strategy   \eqref{eq:attackStrategy1}, then it will use  an expected number of attacks
\begin{equation}\label{eq:numConntam}
    \mathbb{E}[\sum_{t=1}^T \mathbf{1}(\epsilon_t(i_t)\neq 0)]]=O\bigg({(K-1)\log^{\alpha}(T)}/{\mu_{i_A}^{\beta+1}}\bigg),%
\end{equation}
an expected amount of contamination
\begin{equation}
    \mathbb{E}[\sum_{t=1}^T|\epsilon_t(i_t)|]=O\bigg({(K-1)\log^{\alpha}(T)}/{\mu_{i_A}^{\beta+1}}\bigg),
\end{equation}
and it will force   $\mathcal{A}$ to select the action $i_A$ at least $\Omega(T)$ times in expectation, namely $ \mathbb{E}[\sum_{t=1}^T \mathbf{1}(i_t= i_A)]=\Omega(T)$ .
\end{proposition}
 
Proposition \ref{thm:constantAttack} provides a relationship between the regret of the algorithm without attack and the number of attacks (or amount of contamination) sufficient to ensure that the target action $i_A$ is selected $\Omega(T)$ times, which also implies $R^{\mathcal{A}}(T)=\Omega(T)$ if $i_A\neq i^*$. 
Another important consequence of the proposition is that for an order optimal algorithm such as UCB, we have that $\alpha=1$ and $\beta=1$ in \eqref{eq:algo1}. Thus, the expected number of attacks and the expected amount of contamination  are $O(\log T)$.

A small criticism to the attack strategy \eqref{eq:attackStrategy1} might be that it pulls down the reward ``too much''. This turns out to be fixable. In Appendix \ref{append:gap-attack}, we prove  that a different type of attack that pulls the reward of any action $i \not = i_A$ down by an \emph{estimated} gap $\Delta = 2 \max \{  \mu_{i} - \mu_{i_A}, 0 \} $ (similar to the ACE algorithm in \cite{ma2018data}) will also succeed. However, the number of attacks now will be inversely proportional to   $\min_{i\neq i_A}|\mu_i-\mu_{i_A}|^{\beta+1}$, while not $ \mu_{i_A}^{\beta+1}$ as in Proposition \ref{thm:constantAttack}. 

\section{Verification based Algorithms}
\label{sec: verification}

In this section we explore the idea of using verifications to rescue our bandit model from reward contaminations. In particular, we first investigate the case when the amount of verification is not limited, and therefore our main goal is to minimize the number of verifications (along with aiming to restore the order-optimal logarithmic regret bound). We then discuss the case when the number of verifications is bounded above by a budget $B$ (typically of $o(T)$). 

\subsection{Saving Bandits with Unlimited Verifications}
\label{subsec: unlimited verification}

In this setting we assume that the number of verifications is not bounded above, and therefore, our goal is to minimize the number of verifications that is required to restore the logarithmic regret bound. To do so, we first show that any successful verification based algorithm (i.e., they can restore the logarithmic regret) would require $\Omega(\log{T})$ verifications.
In particular, the following theorem  establishes that for all consistent learning algorithm\footnote{A learning algorithm is consistent  \citep{kaufmann2016complexity} if for all $t$,  the action $i_{t+1}$ (a random variable) is measurable given the history $\mathcal{F}_{t}=\sigma (i_1,r^o_1(i_1), i_2,r^o_2(i_2) \ldots, i_{t},r^o_{t}(i_{t}))$. } $\mathcal{A}$  and sufficiently large $T$, 
if the algorithm $\mathcal{A}$ uses $O((\log T)^{1-\alpha})$ verifications with  $0<\alpha < 1$, then the expected regret is $\Omega{((\log T)^{\beta}})$ with $\beta>1$ in the MAB setting with verification.
 \begin{theorem}\label{thm:lowBoundVerification} Let  $KL(i_1,i_2)$ denote the KL divergence between the distributions of actions $i_1$ and $i_2$.  
For all $0<\alpha<1$, $1<\beta$
and all consistent  learning algorithm $\mathcal{A}$, there exists a time $t^*$ and
an attacking strategy such that for all $T\geq 2t^*$ satisfying $(\log T)^{1-\alpha}+\beta\log (4\log T)\leq \log T,$
if the total number of verifications $N^s_T$ until round $T$ is
\begin{equation}\label{eq:boundOnVeri}
    N^s_T<(\log T)^{1-\alpha} /\min_{i_1,i_2\in [K]}KL(i_1,i_2),
\end{equation}
 then the expected regret of $\mathcal{A}$ is at least $\Omega((\log T)^\beta)$. 
 \end{theorem}
Theorem \ref{thm:lowBoundVerification} establishes that $\Omega(\log T)$ verifications are necessary to obtain $O(\log T)$ regret. 
Here, we assume that the number of verifications is bounded above \emph{almost surely}. 
Nevertheless, if instead the \emph{expected} number of verifications is bounded, we shall obtain the following similar bound.
\begin{corollary}
\label{corr:verification expected lower bound} 
For all $0<\alpha<1$, $1<\beta$, all consistent  learning algorithm $\mathcal{A}$ and sufficiently large $T$ such that the requirements in Theorem \ref{thm:lowBoundVerification} are satisfied, there exists
an attacking strategy such that
if the \emph{expected number of verifications} $N^s_T$ until round $T$ is
    $\mathbb{E}[N^s_T]<(\log T)^{1-\alpha} /\min_{i_1,i_2\in [K]}KL(i_1,i_2)$,
 then the expected regret of $\mathcal{A}$ is at least $\Omega((\log T)^\beta)$. 
 \end{corollary}
 
 We now move to design an algorithm that matches this optimal number of verifications.
 Our algorithm is based on the following simple idea: Contamination is only effective when the contaminated reward is used for estimating the mean reward value of the arms, and therefore, influencing the learnt order of the arms. As such, any algorithm that do not need these estimates for most of the time would not suffer much from the contamination if the remaining pulls (when the observed rewards are used for mean estimation) is properly secured via verification. 
 This idea naturally lends us to the explore-then-commit (ETC) type of bandit algorithms~\citep{garivier2016explore}, where in the first phase, the algorithm aims to learn the optimal arm by solving a best arm identification (BAI) problem (exploration phase), and in the second (commit) phase, it just repeatedly pulls the learnt best arm~\citep{kaufmann2016complexity}. 
 It is clear that if the first phase is fully secured (i.e., every single pull within that phase is verified), then we can learn the best arm with high probability, and thus, can ignore the contaminations within the second phase. %
 The choice of the BAI algorithm for the exploration phase is important though. In particular, any BAI with fixed pulling budget would not work here, as they cannot guarantee logarithmic regret bounds~\citep{garivier2016explore}. 
 On the other hand, BAI with fixed confidence will suffice. In particular, we state the following:


\begin{observation}
\label{verification upper bound for Secure-ETC}
Any ETC algorithm, where the exploration phase uses BAI with fixed confidence $\delta = \frac{1}{T}$ and every single pull in that phase is verified, enjoys an expected regret bound of $O\Big(\sum_{i \neq i^*}{\log{T}}/{\Delta_i}\Big)$. In addition, the expected number of verifications is bounded above by $O\Big(\sum_{i \neq i^*}{\log{T}}/{\Delta^2_i}\Big)$.  
\end{observation} 
 
 We refer to the ETC algorithm enhanced with verification described in the above  observation as Secure-ETC. The proof of Observation~\ref{verification upper bound for Secure-ETC} is simple and hence omitted from the main paper. Note that this result, alongside with Theorem \ref{thm:lowBoundVerification}, show that Secure-ETC uses order-optimal number of verification, and enjoys an order-optimal expected regret, irrespective of the attacker's strategy. 
 
 The main drawback of Secure-ETC algorithms is that there is positive probability that the algorithm may keep exploring until the end time $T$. 
 While such small probability event turns out to not be an issue regarding its expected regret, one might prefer another type of algorithm which properly mix the exploration and exploitation.  
 For such interested readers, we propose another algorithm, named Secure-UCB (for Secure Upper Confidence Bound), which integrates verification into the classical UCB algorithm, and also enjoys similar order-optimal regret bounds and order-optimal expected number of verifications. Due to space limitations, we defer both the detailed description of Secure-UCB and its theoretical analysis to the appendix (see Appendix~\ref{appendix:verify with Secure-UCB} for more details). 
 However, for the sake of completeness, we state the following theorem below.
 
 \begin{theorem} \label{thm:SUCB_simple}
For all $T$ such that {$T\geq c_2\log T/\min_{i\neq i^*}\Delta^2(i)$}, Secure-UCB performs $O(\log T)$ number of verification in expectation, and the expected regret of the algorithm is  $O(\log T)$ irrespective of the attacker's strategy. Namely, 
\begin{equation}
\begin{split}
        \sum_{i\in [K]}\mathbb{E}[N^s_T(i)]&\leq
        c_3\big(\sum_{i\neq i^*}{\log T}/{\Delta^2(i)}\big), 
\end{split}
\end{equation}
\begin{equation}
\begin{split}
     R(T)&\leq
        c_4\big(\sum_{i\neq i^*}{\log T}/{\Delta(i)}\big),
\end{split}
\end{equation}
where $N^s_T(i)$ is the total number of verifications for arm $i$ until round $T$ and $c_2$, $c_3$ and $c_4$  are numerical constants (concrete values can be found in the appendix).
\end{theorem}
It is worth noting that due to the sequential nature of UCB, designing a UCB-like algorithm with verification is far from trivial and therefore its technical analysis is significantly more involved.

\subsection{Saving Bandits with Limited Verifications} 
\label{subsec:limited verification}
 
While unlimited verification can completely restore the original regret bounds, we will show next that this is unfortunately not the case if the number of verification are  bounded. In particular, we state this negative result. 

\begin{theorem}\label{thm:LowerBoundFixed Budget}
Consider an attacker with unlimited contamination budget. For any $T$, $K\geq 2$ and $N^s_T\geq K$, if the total number of verifications performed until round $T$ is at most $N^s_T$, then there exists a distribution over the assignment of rewards such  that the expected \emph{gap-independent} regret of any learning algorithm is at least 
\begin{equation}
    R(T)\geq cT\sqrt{K/{N^s_T}}.
\end{equation}
where $c$ is a numerical constant. In addition, for any $T$, $K\geq 2$, and  $N^s_T\geq K$, there exists a distribution over the assignment of rewards such  that the expected  cost, defined as the sum of expected regret and the number of verifications, of any learning algorithm is at least $\Omega(T^{2/3})$. 
\end{theorem}
We remark that the goal of Theorem \ref{thm:LowerBoundFixed Budget} 
is to demonstrate that, unlike the unlimited verification case in subsection \ref{subsec: unlimited verification}, here it is impossible to fully recover from the attack --- in the sense of achieving order optimal regret bounds as in the original bandit setting without attacks --- if $B \in o(T)$, and this motivates our following study (Theorem \ref{thm:Secure-BARBAR regret bound}) of developing regret bounds that scale with the budget $B$. For this purpose, it suffices to have a gap-independent lower bound as in Theorem \ref{thm:LowerBoundFixed Budget}. Nevertheless, we acknowledge that an interesting  research question is to see whether one can achieve a gap-dependent lower bound. This is out of the scope of our current paper and is an independent open question.

\begin{algorithm}[t]
\begin{algorithmic}[1]
\STATE \textbf{Input}: confidences $\beta, \delta \in (0,1)$, time horizon $T$, verification budget $B$
\STATE Set $n^B_i = \Big\lfloor {B}/{K} \Big \rfloor$, $T_0 = B$, $\Delta^0_i = 1$ for all $i \in [K]$, and $\lambda = 1024\ln(\frac{8K}{\delta}\log_2{T})$

\FOR{ epochs $m = 1,2,\dots$}

\STATE Set $n_i^m = \lambda (\Delta_i^{m-1})^{-2}$ for all $i \in [K]$,
$N_m = \sum_{i=1}^{K}n^m_i$, and $T_m = T_{m-1} + N_m$
\FOR{$t= T_{m-1}$ \TO $T_m$  }
\STATE choose arm $i$ with probability $n^m_i/N_m$ and pull it
\STATE if $n^B_i > 0$ then \emph{verify the pull} (i.e., \emph{observe the true reward}), and reduce $n^B_i$ by $1$
\ENDFOR
\STATE Let $S^m_i$ be the total \emph{observed} rewards from pulls of arm $i$ within epoch $m$ (including both verified and unverified ones)

\STATE \textbf{If} $\;$ \emph{all} the pulls of arm $i$ were verified in epoch $m$ \textbf{then} $r^m_i = S^m_i/n^m_i$ 
\STATE \textbf{Else if} $S^m_i/n^m_i \geq \mu_i^B$ \textbf{then}
    $r^m_i = \min \Big\{S^m_i/n^m_i, \mu^B_i + \frac{\Delta^{m-1}_i}{16} + \sqrt{\frac{\ln{2/\beta}}{2n_B}}\Big\}$
\STATE \textbf{Else}
    $r^m_i = \max \Big\{S^m_i/n^m_i, \mu^B_i - \frac{\Delta^{m-1}_i}{16} - \sqrt{\frac{\ln{2/\beta}}{2n_B}}\Big\}$
\STATE Set $r^{m}_{*} = \max_{i}\{r_i^{m} - \Delta_i^{m-1}/16\}$, $\Delta^m_i = \max\{2^{-m}, r^{m}_{*} - r_i^{m}\}$
\ENDFOR
\caption{Secure-BARBAR}
\label{alg:Secure-BARBAR}
\end{algorithmic}
\end{algorithm}

Now, this impossibility result relies on the assumption that the attacker has an unlimited contamination budget (or amount of contamination). 
One might ask what would happen if the attacker is also limited by a contamination budget $C$ as typically assumed in the relevant literature~\citep{gupta2019better,bogunovic2020stochastic,lykouris2018stochastic}.

We now turn to the investigation of this setting in more detail where  contamination budget is at most $C$. To start with, we assume for now that the attacker can only place the contamination before seeing the actual actions of the bandit algorithm. We refer to this type of attackers as \emph{weak} attackers, as opposed to the ones we have been dealing with in this paper (see Section 5 for a comprehensive comparison of different attacker models).  
We describe an algorithm that addresses this case in a provably efficient way. 
In particular, we introduce Secure-BARBAR (Algorithm~\ref{alg:Secure-BARBAR}), which is built on top of the BARBAR algorithm proposed by~\cite{gupta2019better}. The key differences are: (i) Secure-BARBAR sets up a verification budget $n^B_i$ for each arm $i$ and verify that arm until this budget deplets (lines $6-7$); and (ii) use these reward estimate to adjust the estimates ( lines $9-13$). By doing so, we achieve the following result:
\begin{theorem}
\label{thm:Secure-BARBAR regret bound}
With probability at least $1-\delta - \beta$, the regret of Secure-BARBAR against any weak attackers with contamination budget $C$ is bounded by
\begin{equation}
\begin{split}
    &O\bigg(K\min\Big\{C, \frac{T \log{\frac{2}{\beta}}\ln(\frac{8K}{\delta}\log_2{T})}{\sqrt{B/K}}\Big\} \\
    &\qquad \qquad + \sum_{i \neq i^*}\frac{\log{T}}{\Delta_i}\log{\Big(\frac{K}{\delta}\log{T}\Big)}\bigg).
\end{split}
\end{equation}
\end{theorem}
The regret bound is of $\tilde{O}\bigg(\min\Big\{C, { T\log{({2}/{\beta})}}/{\sqrt{B}}\Big\}\bigg)$, which breaks the known $\Omega(C)$ lower bound of the non-verified setting if $C$ is large \cite{gupta2019better}. 

\paragraph{A note on efficient verification schemes against strong attackers.}
In the case of strong attackers, with a careful combination of the idea described in Secure-BARBAR to incorporate the verified pulls into the estimate of the average reward at each round (lines $9-12$ in Algorithm~\ref{alg:Secure-BARBAR}), and the techniques used in the proof of Theorem 1 from ~\cite{bogunovic2020stochastic} \footnote{The key step is to replace Lemma 1 from~\cite{bogunovic2020stochastic} with a verification aware version, using similar ideas applied in the proof of Theorem~\ref{thm:Secure-BARBAR regret bound}.}, we can prove the following result:
With probability at least $1-\delta - \beta$, we can achieve a regret upper bound of $\tilde{O}\bigg(\min\big\{C, { T\log{({2}/{\beta})}}/{\sqrt{B}}\big\}\log{T}\bigg)$. This can be done by modifying the Robust Phase Elimination (RPE) algorithm described in~\cite{bogunovic2020stochastic} with the verification and estimation steps from Algorithm~\ref{alg:Secure-BARBAR}. 
The drawback of this approach is that it only works when the contamination budget $C$ is known in advance. 
Although~\cite{bogunovic2020stochastic} have also provided a method against strong attackers with unknown contamination budget $C$, their method can only achieve $\tilde{O}(C^2)$ under some restrictive constraints (e.g., $C$ has to be sufficiently small). In addition, it is not clear how to incorporate our ideas introduced for Secure-BARBAR to that approach in an efficient way (i.e., to significantly reduce the regret bound from $\tilde{O}(C^2)$). 
Given this, it remains future work to derive an efficient verification method against strong attackers with unknown contamination budget $C$, which can yield regret bounds better than $\tilde{O}(C^2)$.

\section{Comparison of Attacker Models}
\label{sec:comparison of attacker models}

This section provides a more detailed comparison between the different attacker models from the (robust bandits) literature and their corresponding performance guarantees. In particular, at each round $t$, a \emph{weak attacker} has to make the contamination \emph{before} the actual action is chosen. On the other hand, a \emph{strong attacker} can observe both the chosen actions and the corresponding rewards before making the contamination.
From the perspective of contamination budget (or the amount of contamination), it can either be bounded above surely by a threshold, or that bound only holds in expectation. We refer to the former as \emph{deterministic budget}, while we call the latter as \emph{expected budget}.  
To date, the following three attacker models have been studied: (i) weak attacker with deterministic budget; (ii) strong attacker with deterministic budget; and (iii) strong attacker with expected budget.

\paragraph{Weak attacker with deterministic budget.} For this attacker model, \cite{gupta2019better} have proposed a robust bandit algorithm (called BARBAR) that provably achieves $O(KC + (\log{T})^2)$ regret against a weak attacker with (unknown) deterministic budget $C$. They have also proved a matching regret lower bound of $\Omega(C)$. These results imply that in order to successfully attack BARBAR (i.e., to force a $\Omega(T)$ regret), a weak attacker with deterministic budget would need a contamination budget of $\Omega(T)$.

\paragraph{Strong attacker with deterministic budget.}
\cite{bogunovic2020stochastic} have shown that there is a phased elimination based bandit algorithm that achieves $O(\sqrt{T} + C\log{T})$ regret if $C$ is known to the algorithm, and 
$O(\sqrt{T} + C\log{T} + C^2)$ if $C$ is unknown. Note that by moving from the weaker attacker model to the stronger one, we suffer an extra loss in terms of achievable regret (i.e., from $O(C)$ to $O(C^2)$) in case of unknown $C$. While the authors have also proved a matching regret lower bound of $\Omega(C)$ for the known budget case, they have not provided any similar results for the case of unknown budget. Nevertheless, their results show that in order to successfully attack their algorithm, an attacker of this type would need a contamination budget of $\Omega(T)$ for the case of known contamination budget, and $\Omega(\sqrt{T})$ if that budget is unknown.

\paragraph{Strong attacker with expected budget.}
Our Proposition~\ref{thm:constantAttack} shows that this attacker can successfully attack any order-optimal algorithm with a $O(\log{T})$ expected contamination budget (note that~\cite{liu2019data} have also proved a similar, but somewhat weaker result). We have also provided a matching  lower bound on the necessary amount of expected contamination budget against UCB. {It is worth noting that if the rewards are unbounded, then the attacker may use even less amount contamination (e.g., $O(\sqrt{\log{T}})$) to achieve a successful attack~\citep{zuo2020near}.}

\paragraph{{Saving} bandit algorithms with verification.}
The above mentioned results also indicate that if an attacker uses a contamination budget $C$ (either deterministic or expected), the regret that any (robust) algorithm would suffer is $\Omega(C)$. A simple implication of this is that if an attacker has a budget of $\Theta(T)$ (e.g., he can contaminate all the rewards), then no algorithm can maintain a sub-linear regret if they can only rely on the observed rewards. Secure-ETC, Secure-UCB, and Secure-BARBAR break this barrier of $\Omega(C)$ regret with verification. In particular, the former two still enjoy an order-optimal regret of $O(\log{T})$ against any attacker (even when they have $\Theta(T)$ contamination budget) while only using $O(\log{T})$ verifications. The latter, when playing against a weak attacker, still suffers a swift increase in the regret as $C$ is increased. But this increase is not linear in $C$ as in the non-verified setting.

\section{Conclusions}
\label{sec: conclusions}

In this paper we   introduced a reward verification model for bandits to counteract against data contamination attacks. In particular, we contributions can be grouped as follows:
We first revisited the analysis of strong attacker and proved the first attack lower bound of $\Theta(\log{T})$ expected number of contaminations for a successful attack. This lower bound is shown to be tight with our  oblivious attack scheme, the contamination of which matches the lower bound. 
We then move to verification based approaches with unlimited verification, where we first provided two algorithms, Secure-ETC and Secure-UCB, which can recover any attacks with logarithmic number of verifications. We also provided a matching lower bound on the number of verifications. 
For the case of limited verifications, we first showed that full recovery is impossible if the attacked has unlimited contamination budget, unless the verification budget $B = \Theta(T)$. 
In case the attacker is also limited by a budget $C$, we proposed Secure-BARBAR, which achieves a regret lower than the $\Theta(C)$ regret barrier, if used against a weak attacker.

For future research, when facing a strong attacker with contamination budget $C$, we briefly discussed how a similar idea from Secure-BARBAR with limited verification can be used to achieve a regret bound better than $O(C\log{T})$.   
However, this idea requires that $C$ is known in advance. It is  an open question whether for the case of unknown $C$ we can get a similar regret bound that is better than the regret we can achieve for the non-verified case. Second,   
since bounding the contamination in expectation and almost surely leads to different results (see Section \ref{sec:comparison of attacker models}), it would be interesting to study the setting where number of verifications is bounded almost surely. Third, another interesting extension is a \emph{partial feedback verification} model, where the learner can only request a feedback about whether the observed reward is corrupted or not but cannot see the true reward. Finally, extending our study to RL is an intriguing future direction.   

\bibliographystyle{unsrtnat}
\bibliography{references}  





\newpage
\onecolumn
\appendix
\section{Upper Confidence Bound}\label{append:AlgUCB}

\begin{algorithm}
\begin{algorithmic}[h]
\STATE For all $i\in[K]$, initialize $\hat{\mu}_0(i)=0$, $N_0(i)=0$.
\FOR{ $t\leq K$}
\STATE Choose action $i_t=t$, and observe  $r_t(i_t)$.
\STATE Update $\hat{\mu}_{t}(i_t)=r^o_t(i_t)$, $N_{t}(i_t)=N_{t-1}(i_t)+1$ 
\STATE For all $i \not =  i_t$, $\hat{\mu}_{t}(i)=\hat{\mu}_{t-1}(i)$, $N_{t}(i)=N_{t-1}(i)$.
\ENDFOR
\FOR{$K+1\leq t\leq T$}
\STATE Choose action $i_t$  such that
        \begin{equation}\label{eq:UCBCond}
             i_t= \mbox{argmax}_{i\in[K]} \bigg[ \hat{\mu}_{t-1}(i)+ \sqrt{\frac{8\log t}{N_{t-1}(i)}} \bigg] .
        \end{equation}
\STATE Update $N_{t}(i_t)=N_{t-1}(i_t)+1$, and
\begin{equation}
        \hat{\mu}_{t}(i_t)=\frac{\hat{\mu}_{t-1}(i_t)\cdot N_{t-1}(i_t)+ r^o_{t}(i_t)}{N_{t-1}(i_t)+1}.
    \end{equation}
\STATE  For all $i\in [K]\setminus i_t$, $\hat{\mu}_{t}(i)=\hat{\mu}_{t-1}(i)$ and $N_{t}(i)=N_{t-1}(i)$. 
\ENDFOR
\caption{(Classical) Upper Confidence Bound}
\label{alg:UCB}
\end{algorithmic}
\end{algorithm}

\section{Proof of Proposition \ref{thm:constantAttack}}\label{append:const-attack} 
 Let $N_{t}(i)$ be the number of times action $i$ is chosen by the learner until time $t$, namely
\begin{equation}
    N_{t}(i)=\sum_{n=1}^{t}\mathbf{1}(i_n=i).
\end{equation}
Then, we have that
\begin{equation}\label{eq:manipSum}
  \sum_{t=1}^T \mathbf{1}(\epsilon_t(i_t)\neq 0)\leq \sum_{i\neq i_A}N_{T}(i).
\end{equation}
Using \eqref{eq:attackStrategy1}, for all $i\in [K]\setminus i_A$ and $t\leq T$, we have that 
\begin{equation}\label{eq:expectedRewardi}
    \mathbb{E}[r^o_t(i)]=0,
\end{equation}
and 
\begin{equation}\label{eq:expectedRewardiA}
    \mathbb{E}[r^o_t(i_A)]=\mu_{i_A}. 
\end{equation}
Since the algorithm $\mathcal{A}$ makes decision based on the $r_t^o(.)$, using \eqref{eq:algo1}, \eqref{eq:expectedRewardi} and \eqref{eq:expectedRewardiA}, we have that
\begin{equation}\label{eq:regretBasedonObservation}
    \mathbb{E}[T\mu_{i_A}-\sum_{t=1}^Tr^o_t(i_t)]=O\bigg(\frac{(K-1)\log^\alpha(T)}{\mu_{i_A}^\beta}\bigg).
\end{equation}
Also, we have 
\begin{equation}\label{eq:contaminations}
\begin{split}
    \mathbb{E}[T\mu_{i_A}-\sum_{t=1}^{T}r^o_{t}(i_t)]&\stackrel{(a)}{=}\mu_{i_A} \mathbb{E}[\sum_{i\neq i_A} N_T(i)],
\end{split}
\end{equation}
where $(a)$ follows from the fact that $\Delta(i)=\mu_{i_A}$ for the learner. 
This along with \eqref{eq:regretBasedonObservation} implies that
\begin{equation}\label{eq:manipBound1}
   \mathbb{E}[ \sum_{i\neq i_A} N_T(i)]=O\bigg(\frac{(K-1)\log^\alpha(T)}{\mu_{i_A}^{\beta+1}}\bigg).
\end{equation}
Now, we have
\begin{equation}
\mathbb{E}[N_{T}(i_A)]=T-\sum_{i\neq i_A}\mathbb{E}[N_T(i)],
\end{equation}
which  using \eqref{eq:regretBasedonObservation} and \eqref{eq:contaminations}, 
implies the  attack is successful, i.e., $\mathbb{E}[\sum_{t=1}^T \mathbf{1}(i_t= i_A)]=\Omega(T)$.   
Combining \eqref{eq:manipSum} and \eqref{eq:manipBound1}, we have
\begin{equation}
     \mathbb{E}[\sum_{t=1}^T \mathbf{1}(\epsilon_t(i_t)\neq 0)]=O\bigg(\frac{(K-1)\log^\alpha(T)}{\mu^{\beta+1}_{i_A}}\bigg).
\end{equation}
Hence, the statement of the proposition follows. 
\section{Attacks Based on Gap Estimation}\label{append:gap-attack}

The attack is similar to the ACE attack in \cite{ma2018data}. Specifically, the attacker maintains an estimate $\hat{\Delta}_t^A(i,i_A)$ of $\mu_i-\mu_{i_A}$ using the previously selected actions and their rewards, namely
\begin{equation}
    \hat{\Delta}_t^A(i,i_A)= \hat{\mu}_t( i)+\sqrt{\frac{2\log t}{  \tilde N_t( i)}}- \hat{\mu}_t( i_A)+\sqrt{\frac{2\log t}{   \tilde N_t( i_A)}},
\end{equation}
where
\begin{equation}
    \hat{\mu}_t(i)=\frac{\sum_{n=1}^t r_n(i)\mathbf{1}(i_n=i)}{\sum_{n=1}^t\mathbf{1}(i_n=i)}, 
\end{equation}
and $ \tilde{N}_t(i)=\sum_{n=1}^t\mathbf{1}(i_n=i). $
In this attack,  we have 
\begin{equation}\label{eq:attackStrategy2}
    r_t^o(i_t)=\begin{cases} r_t(i_t)-2\max\{0, \hat{\Delta}_t^A(i_t,i_A)\} &\mbox{ if } i_t\neq i_A,\\
    r_t(i_t) 
    &\mbox{ if } i_t= i_A.
    \end{cases}
\end{equation}
This implies that for all $t\leq T$,  the noise added by the attacker is 
\begin{equation}
    \epsilon_t(i_t)=-2\max\{0, \hat{\Delta}_t^A(i_t,i_A)\}\mathbf{1}(i_t\neq i_A).
\end{equation}
In this attack, for all action $i\neq i_A$  such that $\mu_i>\mu_{i_A}$, the attacker forces the expected observed reward to be at most $\mu_{i_A}-(\mu_i-\mu_{i_A})$, and for all action $i\neq i_A$ such that $\mu_i<\mu_{i_A}$, the attacker forces the expected observed reward to be at most $\mu_i$. Therefore, this strategy ensures that the optimal action is $i_A$ based on the observed rewards. The following proposition establishes  the success of the attack, and provides an upper bound on the expected number of contaminations needed by the attacker.

\begin{proposition}\label{thm:GapEstimationBasedAttack}
For any stochastic bandit algorithm $\mathcal{A}$ with expected regret in the absence of attack given by
\begin{equation}\label{eq:algo1_1}
    R^{\mathcal{A}}=O\bigg(\sum_{i\neq i^*}\frac{\log^\alpha(T)}{\Delta^\beta(i)}\bigg),
\end{equation}
where $\alpha\geq 1$ and $\beta\geq 1$; and for any sub-optimal target action $i_A\in [K]\setminus i^*$, if an attacker follows strategy \eqref{eq:attackStrategy2}, then it will use an expected number of attacks 
\begin{equation}
\begin{split}
     \mathbb{E}[\sum_{t=1}^T \mathbf{1}(\epsilon_t(i_t)\neq 0)]=O\bigg(\sum_{i\neq i_A}\frac{\log^\alpha(T)}{|\mu_{i_A}-\mu_i|^\beta (\min_{i^\prime\neq i_A}|\mu_{i^\prime}-\mu_{i_A}|)}\bigg),
\end{split}
\end{equation}
and it will force $\mathcal{A}$ to select the action $i_A$ at least $\Omega(T)$ times in expectation, namely $\mathbb{E}[\sum_{t=1}^T \mathbf{1}(i_t= i_A)]=\Omega(T)$.
\end{proposition}


 \begin{proof} 
 We will use the following lemma. 
 \begin{lemma}\label{lemma:adapAttackBound}
For all $t>K$ and $i\in [K]$, we have that
\begin{equation}
    \mathbb{P}(\hat{\Delta}_t^A(i,i_A)\leq \mu_i-\mu_{i_A})\leq \frac{1}{t^3}.
\end{equation}
 \end{lemma} 
 \begin{proof}
Using Theorem \ref{thm:Hoeffding}, for all $i\in[K]$, we have that
\begin{equation}
\begin{split}
        \mathbb{P}\bigg(\hat{\mu}_t(i)+ \sqrt{\frac{2\log t}{\bar  N_t(i)}}\leq \mu_i\bigg)\leq 1/t^4,
\end{split}
\end{equation}
\begin{equation}
\begin{split}
        \mathbb{P}\bigg(\hat{\mu}_t(i_A)- \sqrt{\frac{2\log t}{\bar N_t(i_A)}}\geq \mu_{i_A}\bigg)\leq 1/t^4.
\end{split}
\end{equation}
This implies that for all $i\in [K]$ and $K<t\leq  T$, we have
\begin{equation}
\begin{split}
      &\mathbb{P}(\hat{\Delta}_t^A(i,i_A)\leq \mu_i-\mu_{i_A})\\
      &\leq  \mathbb{P}\bigg(\hat{\mu}_t(i)+ \sqrt{\frac{2\log t}{\bar  N_t(i)}}\leq \mu_i\bigg)+\mathbb{P}\bigg(\hat{\mu}_t(i_A)- \sqrt{\frac{2\log t}{\bar N_t(i_A)}}\geq \mu_{i_A}\bigg),\\
      &\leq 2/t^4\leq 1/t^3.
\end{split}
\end{equation}
The statement of the lemma follows.
\end{proof}
Now consider the following event
\begin{equation}
    \mathcal{E}=\{\forall i\in[K], \forall t\leq T: \hat{\Delta}_t^A(i,i_A)\geq \mu_i-\mu_{i_A}\}.
\end{equation}
Similar to \eqref{eq:manipSum}, we have that
\begin{equation}\label{eq:manipSum_1}
        \sum_{t=1}^T \mathbf{1}(\epsilon_t(i_t)\neq 0)\leq \sum_{i\neq i_A}N_{T}(i).
\end{equation}
Using \eqref{eq:attackStrategy2}, under event $\mathcal{E}$, for all $i\in [K]\setminus i_A$ such that $\mu_i>\mu_{i_A}$ and $t\leq T$, we have that 
\begin{equation}\label{eq:expectedRewardi_1}
    \mathbb{E}[r^o_t(i)]\leq\mu_i-2(\mu_i-\mu_{i_A})=\mu_{i_A}-(\mu_i-\mu_{i_A}).
\end{equation}
Also, for all $i\in [K]\setminus i_A$ such that $\mu_i<\mu_{i_A}$ and $t\leq T$, we have that 
\begin{equation}\label{eq:expectedRewardi_2}
    \mathbb{E}[r^o_t(i)]
    \leq\mu_i.
\end{equation}
Since the algorithm $\mathcal{A}$ makes decision based on the $r_t^o(.)$, under event $\mathcal{E}$, using \eqref{eq:expectedRewardi_1} and \eqref{eq:expectedRewardi_2}, we have that 
\begin{equation}\label{eq:regretBasedonObservation_1}
    \mathbb{E}[T\mu_{i_A}-\sum_{t=1}^Tr^o_t(i_t)\big|\mathcal{E}]=O\bigg(\sum_{i\neq i_A}\frac{\log^\alpha(T)}{|\mu_{i_A}-\mu_i|^\beta}\bigg).
\end{equation}
Also, we have 
\begin{equation}\label{eq:contaminations_1}
\begin{split}
    \mathbb{E}[T\mu_{i_A}-\sum_{t=1}^{T}r^o_{t}(i_t)\big|\mathcal{E}]&{=}\sum_{i\neq i_A}|\mu_{i_A}-\mu_i| \mathbb{E}[N_T(i)\big|\mathcal{E}]\geq \min_{i\neq i_A}|\mu_i-\mu_{i_A}| \mathbb{E}[\sum_{i\neq i_A} N_T(i)\big|\mathcal{E}].
\end{split}
\end{equation}
Additionally, using Lemma \ref{lemma:adapAttackBound}, we have
\begin{equation}\label{eq:complementProb}
    \mathbb{P}(\Bar{\mathcal{E}})=\sum_{t=1}^T\frac{K-1}{t^3}\leq\frac{\pi^2(K-1)}{6}.
\end{equation}
Now, we have 
\begin{equation}
\mathbb{E}[N_{T}(i_A)]=T-\sum_{i\neq i_A}\mathbb{E}[N_T(i)],
\end{equation}
which  using \eqref{eq:regretBasedonObservation_1}, \eqref{eq:contaminations_1} and \eqref{eq:complementProb}, 
implies $\mathbb{E}[\sum_{t=1}^T \mathbf{1}(i_t= i_A)]=\Omega(T)$.  Combining \eqref{eq:manipSum_1}, \eqref{eq:regretBasedonObservation_1}, \eqref{eq:contaminations_1} and \eqref{eq:complementProb}, we have
\begin{equation}
     \mathbb{E}[\sum_{t=1}^T \mathbf{1}(\epsilon_t(i_t)\neq 0)]=O\bigg(\sum_{i\neq i_A}\frac{\log^\alpha(T)}{|\mu_{i_A}-\mu_i|^\beta (\min_{i^\prime\neq i_A}|\mu_{i^\prime}-\mu_{i_A}|)}\bigg).
\end{equation}
Hence, the statement of the proposition follows.  

\end{proof}

\section{Proof of Theorem \ref{thm:lowerBoundonUCB}}

\label{append:UCB-lower-bound}
The proof crucially relies on the following ``conservativeness'' property of the UCB algorithm. 

\begin{lemma}\label{lemma:MinPulls2}[Restating Lemma \ref{lemma:Min number of pulls}]
Let $t_0$ be such that $  {t_0}/{(\log t_0)^2} \geq 36K^2$. For all $ t \geq t_0$ 
and for any sequence of rewards $\{r^o_n(i)\}_{i\in [K],n\leq t}$ in $[0,1]$ in $[0,1]$ (can even be  adversarial), UCB will select every action $i\in [K]$ at least $ \log (t/2)$ times until round $t$. 
\end{lemma}
\begin{proof} 
Let $N_t(i)$ be the number of times action $i$ is selected until round $t$, namely
\begin{equation}
    N_t(i)=\sum_{n=1}^t\mathbf{1}(i_n=i),
\end{equation}
 and $M_t(i)$ be the number of attacks on action $i$ until round $t$ by the attacker. With slight abuse of notation, we use $[k]$ to denote the set of actions that are pulled strictly less than $  \log (t/2)$ times until round $t$.
 
We prove this lemma by contradiction. 
   Suppose that there exists some time $t \geq t_0$ and $k\leq K$ actions such that for all $i\in [k]$,
   \begin{equation}\label{eq:contr1}
       N_t(i)<\log (t/2).
   \end{equation}

We now divide the time interval $[t/2,3t/4]$   into $k \log (t/2)  $ consecutive blocks of the same length. Thus, the length of each block is    ${t}/{(4k \log (t/2) ) }$.  By the pigeonhole principle, there must exist one block $[t_1,t_3]$ in which we did not select any action in $[k]$, namely 
\begin{equation}\label{eq:intervalSize}
    t_3=t_1+{t}/{4k \log (t/2)  },
\end{equation}
and 
for all $i\in [k]$, we have
\begin{equation}
    N_{t_1-1}(i)=N_{t_3}(i).
\end{equation}

First, we provide a lower bound on the UCB index of all actions  $i \in [k]$ within the time interval (or block) $[t_1,t_3]$. For all $t_2\in [t_1, t_3]$ and $i \in [k]$, we have
\begin{equation}
\label{eq: lowerbound for UCB1}
\hat{\mu}_{t_2}(i) +  \sqrt{\frac{8 \log t_2}{N_{t_2-1}(i)}}   \stackrel{(a)}{>}  \sqrt{\frac{8 \log (t/2)}{\log{(t/2)} }} = 2\sqrt{2}, \forall i \in [k]  
\end{equation}
where $(a)$ follows from the facts that  $t_2 \geq t_1\geq t/2$, and  $N_{t_2-1}(i) \leq N_{t}(i)  <  \log{(t/2)} $ using \eqref{eq:contr1}.  

Second, we show that using the lower bound on the UCB index  in \eqref{eq: lowerbound for UCB1} for actions in $[k]$,  no actions outside $[k]$ can be pulled by more than $8 \log(3t/4) + 1$ times within the interval $[t_1,t_3]$, namely for all $i\in [K] \setminus [k]$, we have
\begin{equation}\label{eq:contr2}
    N_{t_3}(i)-N_{t_1-1}(i)\leq 8 \log(3t/4) + 1.
\end{equation}
We prove \eqref{eq:contr2}  by contradiction. Suppose \eqref{eq:contr2} does not hold. Then, there exists an action $j \in [K] \setminus [k]$ and a time $t_2 \in [t_1,t_3]$ such that action $j$ is selected, namely $i_{t_2}=j$, and 
\begin{equation}\label{eq:numPulls}
    N_{t_2-1}= 8 \log(3t/4) + 1.
\end{equation}
Therefore, at round $t_2$, the UCB index of action $j$ is
\begin{equation}
\label{eq: upperbound for UCB2}
\hat{\mu}_{t_2}(j) + \sqrt{\frac{8 \log t_2}{N_{t_2-1}(j)}}  \stackrel{(a)}{\leq} 
1 + \sqrt{\frac{8\log(3t/4) }{N_{t_2-1}(j)}}   \stackrel{(b)}{<} 2,
 \end{equation}
 where $(a)$ follows from the facts that observed rewards are in the interval $[0,1]$, and $t_2\leq t_3\leq 3t/4$, and $(b)$ follows from \eqref{eq:numPulls}. 
This however is a contradiction since  \eqref{eq: lowerbound for UCB1} shows that the UCB index  of action $i \in [k]$ at time $t_2$ is strictly larger than $2$. Therefore, action $j$ cannot have the the largest UCB index at $t_2$, and thus cannot be selected. Thus, we have that \eqref{eq:contr2}  follows. 

Finally, combining \eqref{eq:contr2} and the fact that actions in the set $[k]$ are not selected in $[t_1,t_3]$, we have that
\begin{equation}
    \sum_{i\in [K]}(N_{t_3}(i)-N_{t_1}(i)) \leq (K-k) (8 \log(3t/4) + 1 ).
\end{equation}
This along with \eqref{eq:intervalSize} implies that
\begin{equation}\label{eq:lowBoun1}
    (K-k) ( 8 \log(3t/4) + 1 ) \geq \frac{t}{4k \log (t/2)}, \text{ or equivalently } \,   4k(K-k)    \geq \frac{t}{ \log (t/2) (8 \log(3t/4) + 1)}.
\end{equation}
We also have
\begin{equation}
 \frac{t}{\log (t/2) ( 8 \log(3t/4) + 1 \big)} \geq  \frac{t}{\log (t) ( 9 \log(t)) } \stackrel{(a)}{\geq}  \frac{t_0}{9 (\log (t_0))^2}  \stackrel{(b)}{\geq}    4K^2,
\end{equation}\label{eq:lowBoun2}
where $(a)$ follows from the facts that $t\geq t_0$ and $t/(\log t)^2$ is an increasing function of $t$, and $(b)$ follows from the assumption the lemma that ${t_0}/{(\log (t_0))^2} \geq 36K^2$. Thus, since $k\geq 1$, we have that \eqref{eq:lowBoun1} and \eqref{eq:lowBoun1} contradict each other. Thus, the interval $[t_1,t_3]$ does not exists, which implies \eqref{eq:contr1} does not hold.

Note that in our proof we did not make any assumption about the sequence of the rewards, except that they are bounded within $[0,1]$. Therefore, it holds for arbitrary reward sequence. 
\end{proof}


For the remainder  of the theorem's proof, we will use the definition of $t_0$ from Lemma~\ref{lemma:Min number of pulls}. For all $0<\epsilon<1$ and $\alpha > 0$ such that $0<\epsilon\alpha\leq 1/2$, and for all $T  > \max\{(t_0)^{\frac{1}{1-\alpha \epsilon}}, \exp{(4^\alpha)}\}$, we have that
\begin{equation}\label{eq:lowBound1}
\begin{split}
    \log T &\geq 4^{\alpha}\stackrel{(a)}{\geq} (1-\alpha \epsilon)^{-\alpha/(\alpha \epsilon)},
\end{split}
\end{equation}
where $(a)$ follows from the fact that using $\epsilon\alpha\leq1/2$, we have $4 \geq (1-\alpha \epsilon)^{-1/(\alpha \epsilon)}$. Using \eqref{eq:lowBound1}, we have that
\begin{equation}\label{eq:bound2}
     (\log T)^{1-\epsilon} \leq (1-\alpha \epsilon)\log T. 
\end{equation}

Let $\hat{\mu}_t(i)$ is the empirical mean of action $i$  at time $t$ using all the observed rewards (including the contaminated ones), namely
\begin{equation}
    \hat{\mu}_t(i)=\frac{\sum_{n=1}^tr^o_{n}(i)\mathbf{1}(i_n=i)}{N_t(i)}
\end{equation}
where $N_t(i)=\sum_{n=1}^t\mathbf{1}(i_n=i)$. Let $\hat{\mu}^m_t(i)$ denote the empirical mean of action $i$  at time $t$ using the true reward  without contamination, namely 
\begin{equation}
    \hat{\mu}^m_t(i)=\frac{\sum_{n=1}^tr_{n}(i)\mathbf{1}(i_n=i)}{N_t(i)}
\end{equation}
Also, let 
$c_t(i) = \sum_{n = 1}^t |\epsilon_n(i_n)| \mathbf{1}(i_n = i)$ denote the total amount of contamination on action $i$ until time $t$. Thus, we have 
\begin{equation}\label{eq:boundMeans}
     \hat{\mu}^m_t(i) - \frac{c_t(i)}{N_{t}(i)} \leq \hat{\mu}_t(i) \leq \hat{\mu}^m_t(i) + \frac{c_t(i)}{N_{t}(i)}. 
\end{equation}


By our hypothesis in the theorem statement, for all $t\leq T$ and $i\in [K]$, we have that 
\begin{equation}\label{eq:boundContaminations}
    c_t(i) \leq (\log T)^{1-\epsilon}.
\end{equation}

We now will examine the UCB index of all actions for any time $t > 2T^{1-\alpha \epsilon}$, which is at least $t_0$ by our choice of $T$.
Using Lemma \ref{lemma:MinPulls2}, for all $t >2 T^{1-\alpha \epsilon}\geq  t_0$ and $i\in [K]$, we have that  
\begin{equation}\label{eq:lbNumPulls}
    N_{t}(i)\geq \log(t/2). 
\end{equation}
Using \eqref{eq:bound2} and the fact that $t >2 T^{1-\alpha \epsilon}$, we also have
\begin{equation}\label{eq:lbParam}
  \log(t/2)   > \log  T^{1-\alpha \epsilon} = (1-\alpha \epsilon) \log T \geq (\log T)^{1-\epsilon}.
\end{equation}

Now, for all $t > 2 T^{1-\alpha \epsilon}$, the UCB index of the optimal action $i^*$ satisfies
\begin{eqnarray} \label{eq: lowerbound for UCB1 with manipulation} \nonumber
\hat{\mu}_{t}(i^*) + \sqrt{\frac{8 \log t}{N_{t}(i^*)}}  & \stackrel{(a)}{\geq} &  \hat{\mu}^m_{t}(i^*) -  \frac{c_t(i^*)}{N_{t}(i^*)}   + \sqrt{\frac{8 \log t}{N_{t}(i^*)}}  \\ \nonumber 
 & \stackrel{(b)}{\geq}  &  \hat{\mu}^m_{t}(i^*) - \frac{(\log T)^{1-\epsilon} }{ N_{t}(i^*) }+ \sqrt{\frac{8 \log t}{N_{t}(i^*)}}  \\ \label{eq:ucb-proof-bound}
 & \stackrel{(c)}{\geq}  &  \hat{\mu}^m_{t}(i^*) - \sqrt{ \frac{(\log T)^{1-\epsilon} }{ N_{t}(i^*) } }+ \sqrt{\frac{8 \log t}{N_{t}(i^*)}}  \\  \nonumber
 & \stackrel{(d)}{\geq} &  \hat{\mu}^m_{t}(i^*) - \sqrt{ \frac{ \log(t/2) }{ N_{t}(i^*) } }+ \sqrt{\frac{8 \log t}{N_{t}(i^*)}}  \\   \nonumber
 & {\geq}  &  \hat{\mu}^m_{t}(i^*) + (\sqrt{8} - 1) \sqrt{\frac{ \log t}{N_{t}(i^*)}},
\end{eqnarray} 
where $(a)$ follows from \eqref{eq:boundMeans}, $(b)$ follows from \eqref{eq:boundContaminations}, $(c)$ follows from the fact that using \eqref{eq:lbNumPulls} and \eqref{eq:lbParam}, we have $(\log T)^{1-\epsilon} / N_{t}(i^*) \leq 1$, and $(d)$ follows from \eqref{eq:lbParam}.

Similarly, we can show that for all  $t > 2 T^{1-\alpha \epsilon}$, the  UCB index of any  sub-optimal action $i\neq i^*$  satisfies
\begin{eqnarray} \label{eq: lowerbound for UCB1 without manipulation} \nonumber
\hat{\mu}_{t}(i) + \sqrt{\frac{8 \log t}{N_{t}(i)}}  & \leq &  \hat{\mu}^m_{t}(i) +  \frac{c_t(i)}{N_{t}(i)}   + \sqrt{\frac{8 \log t}{N_{t}(i)}}  \\  
 & \leq &    \hat{\mu}^m_{t}(i) + (\sqrt{8} + 1) \sqrt{\frac{ \log t}{N_{t}(i)}}  
\end{eqnarray}  

Combining \eqref{eq: lowerbound for UCB1 with manipulation}
and \eqref{eq: lowerbound for UCB1 without manipulation}, 
and using the standard analysis of the UCB algorithm for $2 T^{1-\alpha \epsilon}<t \leq T $, we can show that  sub-optimal action $i\neq i^*$ is pulled  at most $O(\log T/\Delta^2_i)$ times after round $2 T^{1-\alpha \epsilon}$.
Consequently, the total number of times that sub-optimal actions are pulled is at most $2 T^{1-\alpha \epsilon} + O(\log T/\Delta^2_i)$ times, and the regret is upper bound by $2 T^{1-\alpha \epsilon} \max_i\Delta(i) + O(\sum_{i \not = i^*}\log T/\Delta(i))$. Since $2 T^{1-\alpha \epsilon}<T$, the regret of UCB is sub-linear in $o(T)$

\section{Proof of Corollary~\ref{corr:PAC lower bound of attacker for UCB}} 

Let $\delta=(\log T)^{-\epsilon/2}$. Using Markov inequality, and the fact that the expected amount of contamination is at most $(\log{T})^{1-\epsilon}$, we have
\begin{equation}\label{eq:Markov1}
    \mathbb{P}(\sum_{t=1}^T|\epsilon_t(i_t)|\geq(\log T)^{1-\epsilon/2})\leq \frac{(\log{T})^{1-\epsilon}}{(\log T)^{1-\epsilon/2}}= \delta.
\end{equation}
Also, for all $0<\epsilon<1$, there exists a constant $\beta>0$ such that
\begin{equation}\label{eq:betaCondition}
    (\log T)^{-\epsilon/2} T\leq T^{1-\beta},
\end{equation}
or equivalently 
\begin{equation}
    \beta \leq \frac{\epsilon\log \log T}{2\log T}.
\end{equation}
Using Theorem \ref{thm:lowerBoundonUCB} and \eqref{eq:Markov1}, we have that
\begin{equation}
\begin{split}
       &R^{UCB}(T)\\
       &\leq (1-\delta) c_1\big( T^{1-\alpha \epsilon/2} \max_i\Delta(i)+ \sum_{i \not = i^*}\log T/\Delta(i)\big) + \delta T,\\
       &{\leq} c_1\big( T^{1-\alpha \epsilon/2} \max_i\Delta(i)+ \sum_{i \not = i^*}\log T/\Delta(i)\big) + (\log T)^{-\epsilon/2} T,\\
       &\stackrel{a}{\leq} c_1\big( T^{1-\alpha \epsilon/2} \max_i\Delta(i)+ \sum_{i \not = i^*}\log T/\Delta(i)\big)+ T^{1-\beta},
\end{split}
\end{equation}
where $(a)$ follows from \eqref{eq:betaCondition}. Since ${1-\alpha \epsilon}<1$ and $1-\beta<1$, the statement of the theorem follows. 

\section{Proof of Theorem \ref{thm:lowBoundVerification}}\label{sec:lowBoundVerif}
 \begin{proof}
 Consider the following specific attacker strategy   that for all $t$, the attacker always set $r^o_t(i_t)=0$. This implies that if verification is not performed by the algorithm at any round $t$,  then for all $i_1,i_2\in [K]$, we have that 
 \begin{equation}\label{eq:KL1}
     KL(i_1,i_2)=0.
 \end{equation}
 Combining \eqref{eq:KL1} and Theorem 12 in \cite{kaufmann2016complexity}, there exists a constant $t^*$ such that for all $t\geq t^*$, we have 
 \begin{equation}\label{eq:lowBound}
     \mathbb{P}(i_t\neq i^*)\geq \exp{(-\min_{i_1,i_2\in [K]}\mbox{KL}(i_1,i_2)N^s_t)},
 \end{equation}
 where $N^s_t$ is the total number of verifications performed by the learner until round $t$.\footnote{In \cite{kaufmann2016complexity}, $N^s_t$  is the number of observed true rewards collected so far. In our model, since the algorithm only gets $0$ when the reward is not verified, regardless of which action is selected. Such uninformative reward feedback will not make a difference, and the number of observed true rewards in our case is thus precisely $N^s_t$.} 
 
 Now, divide the interval $[T/2,T]$ into $2(\log T)^{1-\alpha} /\min_{i_1,i_2\in [K]}KL(i_1,i_2)$ equal sized intervals. This implies that the size of each interval is $T\min_{i_1,i_2\in [K]}KL(i_1,i_2)/(4(\log T)^{1-\alpha})$. Using the Pigeonhole principle, there exists at least $(\log T)^{1-\alpha}/ \min_{i_1,i_2\in [K]}KL(i_1,i_2)$ intervals such that no verification is performed during these intervals. Let $I=[t_1,t_1+T\min_{i_1,i_2\in [K]}KL(i_1,i_2)/(4(\log T)^{1-\alpha} )]$ denote an interval where no verification is performed. Thus, for all $t\in I$, we have
 \begin{equation}\label{eq:low1}
    \begin{split}
        \mathbb{P}(i_t\neq i^*)&\stackrel{(a)}{\geq} \exp{(-\min_{i_1,i_2\in [K]}KL(i_1,i_2)N^s_{t_1})},\\
        &\stackrel{(b)}{\geq}\exp{(-(\log T)^{1-\alpha})},
    \end{split}
 \end{equation}
 where $(a)$ follows from \eqref{eq:lowBound} and the fact that $t\geq t_1\geq T/2 \geq t^*$, and $(b)$ follows from \eqref{eq:boundOnVeri} and the fact that $N^s_{t_1}\leq N^s_T$. Using \eqref{eq:low1}, the total expected regret in the interval $I$ is at least
 \begin{equation}
 \begin{split}
      &\min_{i\neq i^*}\Delta(i)\sum_{t\in I}\mathbb{P}(i_t\neq i^*)\\ &\stackrel{(a)}{\geq} \min_{i\neq i^*}\Delta(i) \frac{T\min_{i_1,i_2\in [K]}KL(i_1,i_2)}{4(\log T)^{1-\alpha}}\exp{(-(\log T)^{1-\alpha})},\\
 \end{split}
 \end{equation}
 where $(a)$ follows from the fact that the size of interval $I$ is $T\min_{i_1,i_2\in [K]}KL(i_1,i_2)/(4(\log T)^{1-\alpha} )$. Since there are $(\log T)^{1-\alpha}/ \min_{i_1,i_2\in [K]}KL(i_1,i_2)$ intervals with no verification, the regret of the algorithm is at least
 \begin{equation}
 \begin{split}
       &\min_{i\neq i^*}\Delta(i)\sum_{I}\sum_{t\in I}\mathbb{P}(i_t\neq i^*)\\
       &{\geq}\min_{i\neq i^*}\Delta(i) \frac{T}{4} \exp{(-(\log T)^\alpha)},\\
       &\stackrel{(a)}{\geq}\min_{i\neq i^*}\Delta(i) (\log T)^\beta,
 \end{split}
 \end{equation}
 where $(a)$ follows from the fact that $(\log T)^\alpha+\beta\log (4\log T)\leq \log T$. The statement of the theorem follows. 
 \end{proof}

\section{Proof of Corollary~\ref{corr:verification expected lower bound}}

 \begin{proof}
Let $\delta=(\log T)^{-\alpha/2}$. Using Markov's inequality, and the fact that the expected number of verification is at most $(\log{T})^{1-\alpha}/\min_{i_1,i_2\in [K]}KL(i_1,i_2)$, we have
\begin{equation}\label{eq:Markov2}
    \mathbb{P}(N^s_T\geq(\log T)^{1-\alpha/2})\leq \frac{(\log{T})^{1-\alpha}}{(\log T)^{1-\alpha/2}}= \delta.
\end{equation}
Using Theorem \ref{thm:lowBoundVerification} and \eqref{eq:Markov2}, we have that
\begin{equation}
\begin{split}
       &R^{UCB}(T)\\
       &\geq c_3 (1-\delta)  (\log T)^{\beta} ,\\
       &\stackrel{(a)}\geq (1-(\log 2)^{-\alpha/2})(\log T)^{\beta},\\
       &=c_4 (\log T)^{\beta},
\end{split}
\end{equation}
where $c_4$ is a numerical constant, and $(a)$ follows from the fact that $\log T$ is an increasing function of $T$. The statement of the corollary follows.
 \end{proof}

 \section{Proof of Theorem \ref{thm:LowerBoundFixed Budget}}

\begin{proof}
We construct the random distribution of rewards as follows. First, before the algorithm selects any action, an action $I$ is chosen uniformly at random to be the optimal action. The reward of action $I$ is drawn from the distribution $\mbox{Bern}(.5+\epsilon)$, and the reward of other actions is drawn from the distribution $\mbox{Bern}(.5)$. 

There is a man-in-the-middle attacker. If verification is performed by the learning algorithm, then the learner observes the reward it incurs. If the verification is not performed by learning algorithm, then the observed reward is drawn from $\mbox{Bern}(.5)$ for all actions in the set $[K]$. 

Let $\mathbb{P}_i(.)$ and $\mathbb{E}_i[.]$ be the probability and expectation conditioned on action $I=i$ being the optimal action. Also, $\mathbb{P}_{0}(.)$ and $\mathbb{E}_{0}[.]$ be the probability and expectation conditioned on the event that all actions $[K]$ have $\mbox{Bern}(.5)$ reward. Finally, at each round $t$, the learner may or may not choose to perform verification. 

We have that the expected regret is
\begin{equation}
\begin{split}
R(T)&= \frac{1}{K}\sum_{i=1}^K\mathbb{E}_{i}[\sum_{t=1}^T(r_t(i)-r_t(i_t))],\\
&=\frac{1}{K}\sum_{i=1}^K\mathbb{E}_{i}[\epsilon (T-N_T(i))],\\
&=\epsilon(T- \frac{1}{K}\sum_{i=1}^K \mathbb{E}_i[N_T(i)]),
\end{split}
\end{equation}
where $N_T(i)$ is the number of times action $i$ is selected until round $T$.

At round $t$, the action $i_t$ is a function of observed reward history $r^o_{1:t-1}$. We that that
\begin{equation}\label{eq:regret}
\begin{split}
  \mathbb{E}_i[N_T(i)] -  \mathbb{E}_0[N_T(i)]&=\sum_{t}(\mathbb{P}_i(i_t=i)- \mathbb{P}_0(i_t=i))\\
  &\leq \sum_{t: \mathbb{P}_i(i_t=i)\geq \mathbb{P}_0(i_t=i)}(\mathbb{P}_i(i_t=i)- \mathbb{P}_0(i_t=i)) \\
  &= \frac{T}{2}||\mathbb{P}_i-\mathbb{P}_0||_1\\
      &\leq \frac{T}{2} \sqrt{ 2 D_{KL}(\mathbb{P}_{0}(r^o_{1:T})||\mathbb{P}_{i}(r^o_{1:T}))},
\end{split}
\end{equation}
where the last inequality follows from Pinsker's inequality.
Now, we have
\begin{equation}
\begin{split}
    D_{KL}(\mathbb{P}_{0}(r^o_{1:T})||\mathbb{P}_{i}(r^o_{1:T}))&=\sum_{t=1}^T D_{KL}(\mathbb{P}_{0}(r^o_{t}|r^o_{1:t-1})||\mathbb{P}_{i}(r^o_{t}|r^o_{1:t-1}))).\\
\end{split}
\end{equation}
Now, using the chain rule for relative entropy, we have
\begin{equation}
\begin{split}
        &D_{KL}(\mathbb{P}_{0}(r^o_{t}|r^o_{1:t-1})||\mathbb{P}_{i}(r^o_{t}|r^o_{1:t-1})\\
        &= \mathbb{P}_{0}(i_t\neq i)  D_{KL}(0.5||0.5)\\
        &\quad+ \mathbb{P}_{0}(i_t= i \mbox{ and verification was performed at round $t$})D_{KL}(0.5||0.5+\epsilon)\\
       &\quad +\mathbb{P}_{0}(i_t=  i \mbox{ and verification was not performed at round $t$}) D_{KL}(0.5||0.5)\\
        &=-0.5\log(1-4\epsilon^2) \mathbb{P}_{0}(i_t= i \mbox{ and verification was performed at round $t$}). 
\end{split}
\end{equation}
This implies
\begin{equation}
    \begin{split}
   D_{KL}(\mathbb{P}_{0}(r^o_{1:T})||\mathbb{P}_{i}(r^o_{1:T})) &=-0.5\log(1-4\epsilon^2) \mathbb{E}_{0}[N^s_T(i)],
\end{split}
\end{equation}
where $N^s_T(i)$ is the number of times action $i$ was selected when verification was performed until round $T$. Finally, we have 
\begin{equation}\label{eq:regret1}
\begin{split}
     \mathbb{E}_i[N_T(i)] &\leq \mathbb{E}_0[N_T(i)]+\frac{T}{2} \sqrt{-\log(1-4\epsilon^2) \mathbb{E}_{0}[N^s_T(i)] }\\
     &\leq \mathbb{E}_0[N_T(i)]+2\epsilon T\sqrt{\mathbb{E}_{0}[N^s_T(i)]\log(4/3)},
\end{split}
\end{equation}
since $-\log(1-x)\leq 4x\log(4/3)$. if $x\in [0,1/4]$. Combining \eqref{eq:regret} and \eqref{eq:regret1}, we have
\begin{equation}\label{eq:regretLowerBound}
\begin{split}
    R(T)&\geq \epsilon T-\frac{\epsilon}{K} \sum_{i=1}^K\bigg(\mathbb{E}_0[N_T(i)]+2\epsilon T\sqrt{\mathbb{E}_{0}[N^s_T(i)]\log(4/3)})\bigg),\\
    &\stackrel{(a)}{\geq} \epsilon T-\frac{\epsilon T}{K} - \frac{2\epsilon^2 T}{K}\sqrt{K N^s_T\log(4/3)}), 
\end{split}
\end{equation}
where  $(a)$ follows from the fact that $\sum_{k=1}^K\mathbb{E}_0[N_T(i)]=T$ and $\sum_{k=1}^K\sqrt{\mathbb{E}_0[N^s_T(i)]}\leq \sqrt{K N^s_T}$.  For $\epsilon=0.25\sqrt{K/N^s_T}$, the right hand side of \eqref{eq:regretLowerBound} is maximized, and we have
\begin{equation}
    R(T)= \Omega\big(T\sqrt{K/{N^s_T}}\big). 
\end{equation}
\end{proof}

\section{Proof of Theorem~\ref{thm:Secure-BARBAR regret bound}}


\begin{proof}[Proof of Theorem~\ref{thm:Secure-BARBAR regret bound}]
To start with, we introduce some new notations. 
In particular, let $A^m_i$ denote the sum of the true rewards from pulls of arm $i$ within epoch $m$ of Phase 2. 
Let $o^m_i = S^m_i/n^m_i$, and
let $\tilde{n}^m_i$ denote the number of times Secure-BARBAR pulls arm $i$ in epoch $m$.

\paragraph{Step 1: Bounding $|r^m_i - \mu_i|$.}
Consider a single arm $i$. At each epoch $m$, we have 3 possible cases: (i) all the pulls of arm $i$ are verified; (ii) some of them are verified and some of them are not (due to exceeding the verification budget $n^B_i$ of that arm); and (iii) none of the pulls are verified.

To strart with, we consider the third case, when none of the pulls are verified in epoch $m$. This means that all the verifications must have happened before $m$.
By using the standard Chernoff-Hoeffding bound, we get that
\begin{equation}
\label{eq: bound Phase 1}
\mathrm{Pr}\left(|\mu^B_i - \mu_i| \geq \sqrt{\frac{\ln{2/\beta}}{2n_B}} \right) \leq \beta
\end{equation}

Following the proof of Lemma 4 in~\cite{gupta2019better}, we can also prove that for any fixed $m,i$ and $\gamma \geq 4e^{-\lambda/16}$, we have that 
\begin{equation}
\label{eq: bound true reward Phase 2}
\mathrm{Pr}\left(\left|\frac{A^m_i}{n^m_i} - \mu_i \right| \geq \sqrt{\frac{3\ln{4/\gamma}}{n^m_i}} \right) \leq \gamma/2,
\end{equation}
\begin{equation}
\label{eq: bound corruption Phase 2}
\mathrm{Pr}\left(\left|\frac{S^m_i- A^m_i}{n^m_i}\right| \geq \frac{2C_m}{N_m} + \sqrt{\frac{\ln{4/\gamma}}{16n^m_i}} \right) \leq \gamma/2
\end{equation}
where $C_m$ is the maximum amount of contaminations the attacker places onto a single arm within epoch $m$, and
\begin{equation}
\label{eq: bound number pf pulls Phase 2}
\mathrm{Pr}\left(\tilde{n}^m_i \geq 2n^m_i \right) \leq \gamma
\end{equation}
Now, by setting $\gamma = \delta/(2K\log_2(T))$, and using the argument in Lemma 3 of ~\cite{gupta2019better}, we can prove that with at least $1-\beta-\delta$ probability, the following events hold together for all $i,m$:
\begin{align}
    \label{eq: good events Secure-BARBAR}
    |\mu^B_i - \mu_i| &\leq \sqrt{\frac{\ln{2/\beta}}{2n_B}} \\
    \left|\frac{A^m_i}{n^m_i} - \mu_i \right| &\leq \sqrt{\frac{3}{4}}\frac{\Delta^{m-1}_i}{16} < \frac{\Delta^{m-1}_i}{16} \\
    \left|\frac{S^m_i}{n^m_i} - \mu_i \right| &\leq \frac{2C_m}{N_m} +\frac{\Delta^{m-1}_i}{16} \\
    \tilde{n}^m_i &\leq 2n^m_i
\end{align}
We denote this series of events to be $\mathcal{E}$. Note that within $\mathcal{E}$, the true mean reward $\mu_i$ is within the range $\sqrt{\frac{\ln{2/\beta}}{2n_B}}$ of $\mu^B_i$ (the average of the verified rewards from Phase 1). Moreover, $A^m_i/n^m_i$ is also within the range of $\sqrt{\frac{\ln{2/\beta}}{2n_B}} + \frac{\Delta^{m-1}_i}{16}$ of $\mu^B_i$ .
Now we only focus on $\mathcal{E}$.
Recall the definition of $r^m_i$, which is if $S^m_i/n^m_i \geq \mu_i^B$ then   \begin{equation}
r^m_i = \min \Big\{S^m_i/n^m_i, \mu^B_i + \frac{\Delta^{m-1}_i}{16} + \sqrt{\frac{\ln{2/\beta}}{2n_B}}\Big\}
\end{equation} 
otherwise
\begin{equation}
 r^m_i = \max \Big\{S^m_i/n^m_i, \mu^B_i - \frac{\Delta^{m-1}_i}{16} - \sqrt{\frac{\ln{2/\beta}}{2n_B}}\Big\}.
\end{equation}
    
\textbf{Case 1:} $S^m_i/n^m_i \geq \mu_i^B$ and $ S^m_i/n^m_i \geq \mu^B_i + \frac{\Delta^{m-1}_i}{16} + \sqrt{\frac{\ln{2/\beta}}{2n_B}}$. In this case, we have that 
\begin{equation}
    |r^m_i - \mu_i| \leq 2\sqrt{\frac{\ln{2/\beta}}{2n_B}} + \frac{\Delta^{m-1}_i}{16}
\end{equation}
as $\mu_i$ is within the range $\sqrt{\frac{\ln{2/\beta}}{2n_B}}$ of $\mu^B_i$. 
On the other hand, we can also use the argument of Lemma 4 from~\cite{gupta2019better} to show that \begin{equation}
    |r^m_i - \mu_i| \leq \left| \frac{S^m_i}{n^m_i} - \mu_i\right| \leq  \frac{2C_m}{N_m} + \frac{\Delta^{m-1}_i}{16}.
\end{equation}
Putting these together we have that 
\begin{equation}
    |r^m_i - \mu_i| \leq  \min\left\{ \frac{2C_m}{N_m}, 2\sqrt{\frac{\ln{2/\beta}}{2n_B}} \right\}+ \frac{\Delta^{m-1}_i}{16}.
\end{equation} 

\textbf{Case 2:} $S^m_i/n^m_i \geq \mu_i^B$ and $ S^m_i/n^m_i < \mu^B_i + \frac{\Delta^{m-1}_i}{16} + \sqrt{\frac{\ln{2/\beta}}{2n_B}}$.
In this case, $r^m_i = S^m_i/n^m_i$ and therefore, we have
\begin{equation}
    |r^m_i - \mu_i| = \left| \frac{S^m_i}{n^m_i} - \mu_i\right| \leq  \frac{2C_m}{N_m} + \frac{\Delta^{m-1}_i}{16}.
\end{equation}
Similarly, since  $ r^m_i = S^m_i/n^m_i < \mu^B_i + \frac{\Delta^{m-1}_i}{16} + \sqrt{\frac{\ln{2/\beta}}{2n_B}}$ and $\mu_i \geq \mu_B - \sqrt{\frac{\ln{2/\beta}}{2n_B}}$, we have that 
\begin{equation}
    |r^m_i - \mu_i| \leq 2\sqrt{\frac{\ln{2/\beta}}{2n_B}} + \frac{\Delta^{m-1}_i}{16}
\end{equation}
which also yields 
\begin{equation}
    |r^m_i - \mu_i| \leq  \min\left\{ \frac{2C_m}{N_m}, 2\sqrt{\frac{\ln{2/\beta}}{2n_B}} \right\}+ \frac{\Delta^{m-1}_i}{16}.
\end{equation}
We can use similar arguments to prove that this also holds for $S^m_i/n^m_i < \mu_i^B$.
We now turn to the cases when epoch $m$ contains some verified pulls of arm $i$.
If all the pulls are verified within that epoch, then by the definition of Secure-BARBAR, 
we have $r^m_i = \frac{S^m_i}{n^m_i}$ and all the rewards are verified. This is equivalent to the case when the contamination $C'_m = 0 < C_m$, which implies that  
\begin{equation}
    |r^m_i - \mu_i| = \left|\frac{S^m_i}{n^m_i} - \mu_i \right| \leq \frac{\Delta^{m-1}_i}{16}  \leq  \min\left\{ \frac{2C_m}{N_m}, 2\sqrt{\frac{\ln{2/\beta}}{2n_B}} \right\}+ \frac{\Delta^{m-1}_i}{16}.
\end{equation}
with at least $1-\delta$ probability (which is also at least $1-\delta -\beta$).
Similarly, for the case when a few pulls of arm $i$ are verified, while the rest a not (note that there is only one epoch where this case occurs), we can also think that it is equivalent with the attacker using a smaller contamination budget $C'_m < C_m$. This means that all the arguments above hold for this case as well.

\paragraph{Step 2: Bounding the regret.} 
Similarly to the proof of Theorem 1 from ~\cite{gupta2019better}, we can show that the regret is bounded above by
\begin{align}
\label{eq: Secure-BARBAR regret 1}
   R(T) \leq B + 2\sum_{m=1}^{M}\sum_{i=1}^{K}\Delta_in^m_i
\end{align}
where the $B$ term comes from the conservative regret estimation of Phase 1 (we can use $\frac{B}{K}\sum_{i=1}^{K}\Delta_i$ if we want to be more precise).
Note that $M$ here denotes the number of epochs in Phase 2, and it can be easily proven that $M \leq \frac{\log_2(2T)}{2}$.
Still using the arguments from Theorem 1 from ~\cite{gupta2019better}, the second term of Eq.~\eqref{eq: Secure-BARBAR regret 1} can be further bounded by
\begin{align}
\label{eq: Secure-BARBAR regret 2}
   \nonumber
   R(T) &\leq B +  32^2\lambda\sum_{i}\frac{\log_2(T)}{\Delta_i} + 8\lambda\sum_{i \neq i^*}\sum_{m=1}^{M}2^{2m} \left( \sum_{s=1}^{m-1}\frac{1}{8^{m-1-s}}\min \left\{ \frac{2C_s}{N_s}, 2\sqrt{\frac{\ln{2/\beta}}{2n_B}}
   \right\}\right) \\
   &\leq B +  32^2\lambda\sum_{i}\frac{\log_2(T)}{\Delta_i} + 8\lambda\sum_{i \neq i^*}\sum_{s=1}^{M}\min \left\{ \frac{2C_s}{N_s}, 2\sqrt{\frac{\ln{2/\beta}}{2n_B}}
   \right\} \sum_{m=s}^{M}\frac{2^{2m}}{8^{m-1-s}}
\end{align}
We now investigate the inner sum of the last term. In particular, we have that 
\begin{equation}
    \sum_{s=1}^{M}\frac{2C_s}{N_s} \sum_{m=s}^{M}\frac{2^{2m}}{8^{m-1-s}} \leq 2\sum_{s=1}^{M}C_s\frac{16}{\lambda} \leq \frac{32C}{\lambda}
\end{equation}
where $C$ is the total contamination budget. This equation holds due Lemma 2 from ~\cite{gupta2019better} which states that $N_s \geq \lambda 2^{2s-2}$.
On the other hand, 
\begin{align}
      2\sqrt{\frac{\ln{2/\beta}}{2n_B}}
\sum_{m=s}^{M}\frac{2^{2m}}{8^{m-1-s}}   \leq \sqrt{\frac{\ln{2/\beta}}{2n_B}}
2^{2s+5}
\end{align}
That is, 
\begin{equation}
    \sum_{s=1}^{M}2\sqrt{\frac{\ln{2/\beta}}{2n_B}}
\sum_{m=s}^{M}\frac{2^{2m}}{8^{m-1-s}}   \leq \sqrt{\frac{\ln{2/\beta}}{2n_B}}
2^{2M+6} \leq 128\sqrt{\frac{\ln{2/\beta}}{2n_B}}T
\end{equation}
The last inequality is due to $M \leq \frac{\log_2(2T)}{2}$.
Putting all these together, and considering $n_B = \lfloor B/K \rfloor \geq B/2K$ for sufficiently large $B$ (e.g., $B \geq 2K$), we have that
\begin{align}
\label{eq: Secure-BARBAR regret 3}
   \nonumber
   R(T) &\leq B +  32^2\lambda\sum_{i}\frac{\log_2(T)}{\Delta_i} + \min\left\{256KC, 1024\lambda KT \sqrt{\frac{\ln(2/\beta)}{B/K}} \right\}
\end{align}
We conclude the proof by plugging the definition of $\lambda$ into the equation above.
\end{proof}

\section{Verification with Secure Upper Confidence Bound}
\label{appendix:verify with Secure-UCB}

\begin{algorithm}[H]
\algsetup{linenosize=\tiny}
  \small
\begin{algorithmic}[1]
\STATE For all $i\in[K]$, initialize $B=[K]$, $\hat{\mu}_0(i)=0$, $N_0(i)=0$,  $t=1$.
\FOR{ $t\leq K$}
\STATE Choose action $i_t=t$.
\STATE Verify the observed reward, i.e., $r^o_t(i_t)=r_t(i_t)$.
\STATE Update $\hat{\mu}_{t}(i_t)=r_t(i_t)$, and $N_{t}(i_t)=N_{t-1}(i_t)+1$. 
\STATE For all $i\in [K]\setminus i_t$, $\hat{\mu}_{t}(i)=\hat{\mu}_{t-1}(i)$ and  $N_{t}(i)=N_{t-1}(i)$.
\ENDFOR
\STATE  For all $ i\in[K]$, update $\hat{\Delta}_{K}^*$  in \eqref{eq:delta_min}.
\FOR{$K+1\leq t\leq T$}
\STATE Choose action $i_t$ such that
        \begin{equation}\label{eq:actionSelection}
             i_t= \mbox{argmax}_{i\in[K]} \big( \hat{\mu}_{t-1}(i)+ \sqrt{{400\log \ci}/{N_{t-1}(i)}} \big).
        \end{equation}
\IF{$N_{t-1}(i_t)\leq 1200\log \ci/\hat\Delta^{*2}_{t-1}$}
\STATE Verify the observed reward, i.e., $r^o_t(i_t)=r_t(i_t)$.\\
\STATE Update $N_{t}(i_t)=N_{t-1}(i_t)+1$, and 
\begin{equation}\label{eq:updatemu}
        \hat{\mu}_{t}(i_t)=\frac{\hat{\mu}_{t-1}(i_t)\cdot N_{t-1}(i_t)+ r^o_{t}(i_t)}{N_{t-1}(i_t)+1},   
    \end{equation}
\STATE Update  $\hat\Delta^*_{t}$ in \eqref{eq:delta_min},  $\forall i\in[K]$.
   
\ELSE
\STATE Update $\hat{\mu}_{t}(i_t)=\hat{\mu}_{t-1}(i_t)$ and $N_{t}(i_t)=N_{t-1}(i_t)$.
\ENDIF
\STATE  For all $i\in [K]\setminus i_t$, $\hat{\mu}_{t}(i)=\hat{\mu}_{t-1}(i)$ and $N_{t}(i)=N_{t-1}(i)$.
\ENDFOR
\caption{Secure Upper Confidence Bound}
\label{alg:SUCB}
\end{algorithmic}
\end{algorithm}


In this section, we propose the \emph{Secure Upper Confidence Bound} (Secure-UCB) algorithm {which utilizes verification, and } is robust to \emph{any} data poisoning attack. Specifically, Secure-UCB uses only $O(\log T)$   reward verifications and exhibits $O(\log T)$ regret, \emph{irrespective of the amount of contamination and the number of attacks }. Moreover, we prove that $\Omega(\log T)$ verifications are necessary for any bandit algorithm to have  $O(\log T)$ regret. Therefore, Secure-UCB uses an order-optimal number of verifications $O(\log T)$, and guarantees the order-optimal regret   $O(\log T)$.    

The details of Secure-UCB are presented in Algorithm \ref{alg:SUCB}. At each round $t\leq K$, Secure-UCB  selects an action $i\in[K]$ in round-robin manner, verifies all the reward observations  and updates the corresponding parameters, see Algorithm \ref{alg:SUCB}. At each round $t> K$, Secure-UCB  selects an action $i_t$ with the largest upper confidence bound of similar format as the classical UCB. However, Secure-UCB differs from  UCB  in the following three crucial aspects:
\begin{enumerate}
    \item The confidence interval  $ \hat{\mu}_{t}(i)+ \sqrt{8\log t/N_{t}(i)} $ of the classical UCB algorithm depends on the total number of rounds the action $i$ is \emph{selected} until round $t$, namely $N_{t}(i)$. However, in Secure-UCB,   the confidence interval uses the total number of rounds the action $i$ is \emph{verified} until round $t$, namely $N^s_t(i)$. Note that, like classical UCB, Secure-UCB also uses the empirical mean $\hat{\mu}_{t}(i)$ of the \emph{observed} rewards.  
    \item At each round $t$, Secure-UCB takes an additional step to decide whether to verify the reward of the current action $i_t$  or not, based on a carefully designed criterion.
    \item If the algorithm decides to not verify the reward observation $r^o_t(i_t)$, then it will additionally decide, based on another carefully designed criterion,  whether to ignore the current unverified rewards $r^o_t(i_t)$ by \emph{not} updating both the empirical mean $\hat{\mu}_{t-1}(i_t)$ and the number of rounds the current action is selected $N_{t-1}(i_t)$.    
\end{enumerate} 

The first and second deviations from UCB, as described above, are to guarantee that at any round, the algorithm always has the correct and sufficient confidence level. This is done through: (1) using the number of verified rewards for the confidence interval since the algorithm is only certain about these observations; (2) dynamically requesting additional verifications as the algorithm proceeds to guarantee the desirable confidence level.  The third deviation is designed to control the integration of unverified rewards into the empirical mean estimation so that it does not contain too many attacked (or unverified) rewards. 

Next, we give more details about the Secure-UCB algorithm. The first aspect described above of  using number of verifications in the confidence interval is easy to implement, and is presented in \eqref{eq:actionSelection}. Our descriptions below are primarily focused on the second and third key differences.  

\paragraph{Criterion for Performing Verification. } Secure-UCB maintains a count $N_t^s(i)$ of the number of verification  performed until round $t$ for each action $i\in[K]$. Additionally, it also maintains a ``secured'' empirical mean $\hat{\mu}^s_t(i)$, which is the empirical estimate of the mean of action $i$ using all the verified reward observations until round $t$. Secure-UCB uses this mean $\hat{\mu}^s_t(i)$   in the criterion to decide whether to perform additional verification.  Specifically, it performs a verification at round $t$ if the following criterion holds:
\begin{equation}\label{eq:conditionVerification}
    N_{t-1}(i_t)\leq 1200\log \ci/\hat\Delta^{*2}_{t-1},
\end{equation}
where $\hat\Delta^{*}_{t}$ is intuitively the estimation of  $\min_{i\neq i^*}\Delta(i)$, which is the difference between the largest expected reward and the second largest expected reward. In Secure-UCB, this estimation $\hat\Delta^{*}_{t}$  is based on the verified rewards and is defined as the difference between the largest lower confidence bound (obtained by, say action $a^*_t$)  and the largest upper confidence bound among all actions excluding $a^*_t$, namely
\begin{small}
\begin{equation}\label{eq:delta_min}
    \hat\Delta^*_{t}=\max\bigg\{0,\hat{\mu}_t(a_t^*)-\sqrt{\frac{3\log \ci}{ N_t(a_t^*)}}-\hat{\mu}_t(\tilde{a}_t)-\sqrt{\frac{3\log \ci}{  N_t(\tilde{a}_t)}}\bigg\},
\end{equation}
where 
$$ a^*_t=\mbox{argmax}_{a\in [K]} \bigg[ \hat{\mu}_t(a)-\sqrt{{3\log \ci}/{N_t(a)}} \bigg], $$ $$ \tilde{a}_t=\mbox{argmax}_{a\in [K]\setminus a^*} \bigg[ \hat{\mu}_t(a)+\sqrt{{3\log \ci}/{N_t(a)}} \bigg]. $$ 
\end{small}

We show in  Lemma \ref{lemma:DeltaBoundEstimate} in Appendix~\ref{append:proof of Theorem 2} that with high probability, we have $\hat{\Delta}_{t}^*\leq \min_{i\neq i^*}\Delta(i). $ Note that, after verification, the algorithm will observe the true reward, namely $r^o_t(i_t) = r_t(i_t)$ at round $t$. Also, \eqref{eq:conditionVerification}-  \eqref{eq:delta_min} depend on the time horizon $T$. This is for convenience of our analysis --- if $T$ is unknown, the doubling trick can be used in conjunction with Secure-UCB  \cite{besson2018doubling}. 

\paragraph{Order-Optimality of Secure-UCB. } The following theorems establish the upper bound on  both the regret of Secure-UCB and the expected number of verifications performed.   
We restate Theorem~\ref{thm:SUCB_simple} with more details as follows:
\begin{theorem} \label{thm:SUCB}
For all $T$ such that {$T\geq c_2\log T/\min_{i\neq i^*}\Delta^2(i)$}, Secure-UCB performs $O(\log T)$ number of verification in expectation, and the expected regret of the algorithm is  $O(\log T)$ irrespective of the attacker's strategy. Namely, 
\begin{equation}\label{eq:numVerifSUCB}
\begin{split}
        \sum_{i\in [K]}\mathbb{E}[N^s_T(i)]&\leq
        c_3\bigg(\sum_{i\neq i^*}{\log T}/{\Delta^2(i)}\bigg), 
\end{split}
\end{equation}
\begin{equation}\label{eq:regretSUCB}
\begin{split}
     R^{SUCB}(T)&\leq
        c_4\bigg(\sum_{i\neq i^*}{\log T}/{\Delta(i)}\bigg),
\end{split}
\end{equation}
where $c_2$, $c_3$ and $c_4$  are numerical constants whose values can be found in the proof.
\end{theorem}
 Theorem \ref{thm:SUCB} establishes that the regret of Secure-UCB in stochastic bandit setting is $O(\sum_{i\neq i^*}\log (T)/\Delta(i))$  \emph{irrespective of the attacker's strategy}. 
 This regret bound is of the same order as the regret bound of the classical UCB algorithm without attack. This implies that Secure-UCB is order optimal in terms of regret, and is robust to any adversary if it can selectively verify up to $O(\log T)$ reward observations in expectation. 



\section{Proof of Theorem \ref{thm:SUCB}}
\label{append:proof of Theorem 2}
\begin{theorem}\label{thm:Hoeffding}{Hoeffding's inequality:} Let $x_1,\ldots,X_n$ be independent and identically distributed random variable, such that for all $i$, we have $0\leq X_i\leq 1$ and $\mathbb{E}[X_i]=\mu$. Then,
\begin{equation}
    \mathbb{P}\bigg(\frac{\sum_{i=1}^n X_i}{n}-\mu\geq \sqrt{\frac{\log(1/\delta)}{2n}}\bigg)\leq \delta,
\end{equation}
\begin{equation}
    \mathbb{P}\bigg(\mu-\frac{\sum_{i=1}^n X_i}{n}\geq \sqrt{\frac{\log(1/\delta)}{2n}}\bigg)\leq \delta.
\end{equation}
\end{theorem}
\begin{lemma}\label{lemma:DeltaBoundEstimate}
For all $K<t\leq T$ , we have 
\begin{equation}
    \mathbb{P}( \hat{\Delta}^*_{t}> \min_{i\neq i^*} \Delta(i))\leq 2K/\ci^6.
\end{equation}
\end{lemma}
\begin{proof}
Using Theorem \ref{thm:Hoeffding}, for all $i\in[K]$, we have that
\begin{equation}\label{eq:LCB_1}
\begin{split}
        \mathbb{P}\bigg(\hat{\mu}_t(i)- \sqrt{\frac{3\log \ci}{N_t(i)}}\geq \mu_i\bigg)\leq 1/\ci^6,
\end{split}
\end{equation}
\begin{equation}\label{eq:UCB_2}
\begin{split}
        \mathbb{P}\bigg(\hat{\mu}_t(i)+ \sqrt{\frac{3\log \ci}{N_t(i)}}\leq \mu_i\bigg)\leq 1/\ci^6.
\end{split}
\end{equation}
Let $\mu_{i^*}-\mu_{i_1}=\min_{i\neq i^*}\Delta(i)$. Also, for all $K<t\leq T$, let an event
\begin{equation}
    \mathcal{E}(t)=\bigg\{\forall i\in [K]: |\hat{\mu}_t(i)-\mu_i|\leq \sqrt{\frac{3\log \ci}{N_t(i)}}\bigg\}.
\end{equation} 
If $\mathcal{E}(t)$ occurs, then we have 
\begin{equation}\label{eq:LCBmax1}
    \hat{\mu}_t(a^*_t)-\sqrt{\frac{3\log \ci}{  N_t(a^*_t)}}\leq \mu_{i^*}.
\end{equation}
Also, if $\mathcal{E}(t)$ occurs, then there exist two action $i^*$ and $i_1$ such that 
\begin{equation}
    \hat{\mu}_t(i^*)+\sqrt{{3\log \ci}/{N_t(i^*)}}\geq \mu_{i_1},
\end{equation}
and
\begin{equation}
    \hat{\mu}^s_t(i_1)+\sqrt{{3\log \ci}/{N_t(i_1)}}\geq \mu_{i_1}.
\end{equation}
This implies that 
\begin{equation}\label{eq:LCBmax2}
    \hat{\mu}_t(\tilde{a}_t)+\sqrt{\frac{3\log \ci}{ N_t(\tilde{a}_t)}}\geq \mu_{i_1}. 
\end{equation}
Using \eqref{eq:LCBmax1} and \eqref{eq:LCBmax2}, if $\mathcal{E}(t)$ occurs, then we have 
 \begin{equation}
        \hat{\Delta}^*_{t}=\hat{\mu}_t(a^*_t)-\sqrt{\frac{3\log \ci}{ N_t(a^*_t)}}-\hat{\mu}_t(\tilde{a}_t)-\sqrt{\frac{3\log \ci}{  N_t(\tilde{a}_t)}}\leq \mu_{i^*}-\mu_{i_1}=\min_{i\neq i^*}\Delta(i). 
    \end{equation}
This implies that
\begin{equation}
\begin{split}
    \mathbb{P}( \Delta^*_{t}> \min_{i\neq i^*} \Delta(i))&\leq \mathbb{P}(\bar{\mathcal{E}}(t))\\
    &\stackrel{(a)}{\leq} \frac{2K}{\ci^6},
\end{split}
\end{equation}
where $\bar{\mathcal{E}}(t)$ denotes the complement of the event ${\mathcal{E}}(t)$, and $(a)$ follows from \eqref{eq:LCB_1} and \eqref{eq:UCB_2}.

\end{proof}
\begin{lemma}\label{lemma:LowerBoundsamples}Let $\mathcal{T}$ be the set of rounds for which verification is not performed. {Let function $f(T)$ be
\begin{equation}\label{eq:ftDef}
    f(T)=\frac{1200 \log T}{\min_{i\neq i^*}\Delta^2(i)}+\sum_{i\neq i^*}\frac{900\log T}{ \Delta^2(i)}+K-1.
\end{equation}
Let $T$ is sufficiently large such that $T\geq f(T)$. Then, for all $f(T)\leq t\leq T$ and for all $i\in [K]$, we have
\begin{equation}
    \mathbb{P}\bigg(N_t(i^*)\leq \frac{1200\log \ci}{\min_{i\neq i^*}\Delta^2(i)}\bigg) \leq \frac{4K^2}{\ci^5},
\end{equation}
and 
\begin{equation}
     \mathbb{P}\bigg(\forall i\neq i^*: N_t(i)\leq \frac{25 \log \ci}{ \Delta^2(i)} \mbox{ or }N_t(i)\geq \frac{900\log \ci}{ \Delta^2(i)}+1\bigg) \leq \frac{4K^2}{\ci^5}.
\end{equation}
Additionally, for all $t\in\mathcal{T}$ such that $K\leq t$, we have that
\begin{equation}
    \mathbb{P}(i_t\neq i^*)\leq \frac{4K^2}{\ci^5}.
\end{equation}
}
\end{lemma}
\begin{proof}
Let $\mathcal{T}$ be  a set of rounds such that for all $t\in \mathcal{T}$, the  action $i_t \in [K]$ satisfies
\begin{equation}
     N_{t-1}(i_t)> \frac{1200\log T}{\hat\Delta^{*2}_{t}}. 
\end{equation}
Consider the following  events 
\begin{equation}
    \mathcal{E}_1(t)=\bigg\{\forall i\in [K]\mbox{ and }\forall K\leq t^\prime\leq t: |\hat{\mu}_{t^\prime}(i)-\mu_i|\leq \frac{1}{2}\sqrt{\frac{400\log \ci}{N_{t^\prime}(i)}}\bigg\},
\end{equation} 
\begin{equation}
    \mathcal{E}_2(t)=\{\forall K\leq t^\prime\leq t: \hat{\Delta}_{t^\prime}^*\leq \min_{i\neq i^*} \Delta(i)\},
\end{equation}

Now, we will show by induction that for all $i\neq i^*$ and $K\leq t\leq T$, if $\mathcal{E}_1(t)$ and $\mathcal{E}_2(t)$  occurs, then we have 
\begin{equation}\label{eq:conc11}
    N_t(i)\leq  \frac{900 \log\ci}{\Delta^2(i)}+1
\end{equation}
We have that \eqref{eq:conc11} trivially holds for $t=K$. Also, if $\mathcal{E}_1(t)$ and $\mathcal{E}_2(t)$  occurs, then $\mathcal{E}_1(t-1)$ and $\mathcal{E}_2(t-1)$  occurs. Now, let \eqref{eq:conc11} holds for $t-1$. If $i_t\neq i$, then \eqref{eq:conc11} holds for $t$ since $\log(.)$ is an increasing function. If $i_t=i$, then we have
\begin{equation}
    \hat{\mu}_{t-1}(i)+\sqrt{\frac{400\log\ci}{N_{t-1}(i)}}\geq \hat{\mu}_{t-1}(i^*)+\sqrt{\frac{400\log\ci}{ N_{t-1}(i^*)}}.
\end{equation}
Under the events $\mathcal{E}_1(t)$ and $ \mathcal{E}_2(t)$, this  implies that
\begin{equation}
    \mu_i+\frac{3}{2}\sqrt{\frac{400\log\ci}{  N_{t-1}(i)}}\geq \mu_{i^*},
\end{equation}
or equivalently
\begin{equation}
    N_{t-1}(i)\leq \frac{900\log \ci}{\Delta^2(i)}.
\end{equation}
This along with $ N_{t}(i)\leq  N^s_{t-1}(i)+1$ implies \eqref{eq:conc11}. 

{Second, we will show that for all $f(T)\leq t\leq T$, we have
\begin{equation}\label{eq:LowerBoundOptimal}
    N_t(i^*)\geq \frac{1200 \log\ci}{\min_{i\neq i^*}\Delta^2(i)}.
\end{equation}
We have that
\begin{equation}
\begin{split}
    N_t(i^*)&=t-\sum_{i\neq i^*} N_t(i),\\
    &\stackrel{(a)}{\geq} f(T)-\sum_{i\neq i^*} N_t(i),\\
    &\stackrel{(b)}{\geq} f(T)-\sum_{i\neq i^*}\frac{900\log \ci}{\Delta^2(i)}-K+1,\\
    &\stackrel{(c)}{\geq} \frac{1200 \log\ci}{\min_{i\neq i^*}\Delta^2(i)},
\end{split}
\end{equation}
where $(a)$ follows from the fact that $f(T)\leq t$, $(b)$ follows \eqref{eq:conc11}, and $(c)$ follows from the definition of $f(T)$.}

Third, we will show by induction that for all $i\neq i^*$ and $k\leq t\leq T$, we have 
\begin{equation}\label{eq:conc22}
    N_t(i)\geq  \frac{1}{16}\min\bigg(\frac{400\log \ci}{\Delta^2(i)} , \frac{4}{9}(N_t(i^*)-1)\bigg)
\end{equation}
Similar to \eqref{eq:conc11}, we only need to show that \eqref{eq:conc22} holds if $i_t=i^*$. If $i_t=i^*$, then we have
\begin{equation}
      \hat{\mu}_{t-1}(i^*)+\sqrt{\frac{400\log\ci}{ N_{t-1}(i^*)}}\geq\hat{\mu}_{t-1}(i)+\sqrt{\frac{400\log\ci}{ N_{t-1}(i)}},
\end{equation}
which implies that under the events $\mathcal{E}_1(t)$ and $ \mathcal{E}_2(t)$, 
\begin{equation}
    \mu_{i^*}+\frac{3}{2}\sqrt{\frac{400\log\ci}{N_{t-1}(i^*)}}\geq  \mu_{i}+\frac{1}{2}\sqrt{\frac{400\log\ci}{ N_{t-1}(i)}}.
\end{equation}
Thus, we have 
\begin{equation}
    N_{t-1}(i)\geq \frac{1}{4}\frac{400\log\ci}{\bigg(\Delta(i) + \frac{3}{2}\sqrt{\frac{400\log\ci}{N_{t-1}(i^*)}}\bigg)^2}\geq \frac{1}{4}\frac{400\log\ci}{\bigg(2\max\{\Delta(i), \frac{3}{2}\sqrt{\frac{400\log\ci}{N_{t-1}(i^*)}}\}\bigg)^2}. 
\end{equation}
This along with $ N_{t}(i)\leq  N_{t-1}(i)+1$  and $ N_{t}(i^*)\leq  N_{t-1}(i^*)+1$ implies \eqref{eq:conc22}. 

{Now, combining \eqref{eq:LowerBoundOptimal} and \eqref{eq:conc22}, under the events $\mathcal{E}_1(t)$ and $ \mathcal{E}_2(t)$, for all  $f(T)\leq t\leq T$ and $i\neq i^*$,  we have that 
\begin{equation}\label{eq:iLowBound}
    N_t(i)\geq \frac{25 \log\ci}{\Delta^2(i)}.
\end{equation}
}

Let $\mathcal{E}_1(t)$ and $\mathcal{E}_2(t)$ occurs. Then,
for all $t\in \mathcal{T}$, we have
\begin{equation}\label{eq:iprimeUpperBound}
   N_{t-1}(i_t)= N_{t}(i_t)> \frac{1200\log \ci}{\hat\Delta^{*2}_{t}}\geq \frac{1200\log \ci}{\min_{i\neq i^*}\Delta^2(i)}.
\end{equation} 
Using \eqref{eq:iprimeUpperBound} and \eqref{eq:conc11}, under the events $\mathcal{E}_1(t)$ and $ \mathcal{E}_2(t)$, for all $t\in \mathcal{T}$, we have that 
\begin{equation}\label{eq:actionChoosen}
    i_t=i^*,
\end{equation}
which implies 
\begin{equation}
    N_t(i^*)\geq \frac{1200 \log\ci}{\min_{i\neq i^*}\Delta^2(i)}.
\end{equation}

Also, using Hoeffding's inequality and Lemma \ref{lemma:DeltaBoundEstimate}, we have
\begin{equation}\label{eq:probComplement1}
\begin{split}
   \mathbb{P}(\bar{\mathcal{E}}_1(t)) + \mathbb{P}(\bar{\mathcal{E}}_2(t)) 
    &\leq \sum_{t^\prime=K}^t\frac{K}{\ci^{ 6}} + \frac{2K^2}{\ci^{ 6}} \leq\frac{4K^2}{\ci^5}.
\end{split}
\end{equation}
Combining \eqref{eq:conc11},  \eqref{eq:LowerBoundOptimal},  \eqref{eq:iLowBound}, \eqref{eq:actionChoosen} and \eqref{eq:probComplement1}, the statement of the lemma follows. 
\end{proof}

\begin{lemma}\label{lemma:8}
For all $f(T)\leq t\leq T$, we have
\begin{equation}
    \mathbb{P}(  \hat\Delta^*_{t}\leq 0.2 \min_{i\neq i^*}\Delta(i))\leq 10K^2/\ci^5. 
\end{equation}
\end{lemma}
\begin{proof}
Let $\min_{i\neq i^*}\Delta(i)=\mu_{i^*}-\mu_{i_1}$. Now, 
consider the following events
\begin{equation}
    \mathcal{E}_1(t)=\bigg\{\forall i\in [K]: |\hat{\mu}_t(i)-\mu_i|\leq \sqrt{\frac{3\log \ci}{N_t(i)}}\bigg\},
\end{equation}
\begin{equation}
     \mathcal{E}_2(t)=\bigg\{\forall i\in[K]\setminus i^*:\frac{25\log \ci}{ \Delta^2(i)}\leq N_{t}(i)\bigg\},
\end{equation}
\begin{equation}
    \mathcal{E}_3(t)=\bigg\{ \frac{1200 \log\ci}{\min_{i\neq i^*}\Delta^2(i)}\leq N_{t}(i^*)\bigg\}. 
\end{equation}
Under events $ \mathcal{E}_1(t)$, $ \mathcal{E}_2(t)$ and $ \mathcal{E}_3(t)$,  for all $i\neq i^*$ and $f(T)\leq t\leq T$, we have that
\begin{equation}\label{eq:uppBound1}
     |\hat{\mu}_t(i)-\mu_i|\leq \sqrt{\frac{3\log \ci}{N_t(i)}} < 0.35\Delta(i),
\end{equation}
and
\begin{equation}\label{eq:uppBound2}
     |\hat{\mu}_t(i^*)-\mu_{i^*}|\leq \sqrt{\frac{3\log \ci}{N_t(i^*)}}\leq \min_{i\neq i^*}0.05\Delta(i). 
\end{equation}
This implies that for all $i\in [K]$ and $t\geq t_*$, we have
\begin{equation}\label{eq:deltahalf}
    |\hat{\mu}_t(i)-\mu_i|< \Delta(i)/{2}.
\end{equation}
Since $ a^*_t=\mbox{argmax}_{a\in [K]}\hat{\mu}_t(a)-\sqrt{{3\log \ci}/{N_t(a)}}$ and $ \tilde{a}_t=\mbox{argmax}_{a\in [K]\setminus a^*}\hat{\mu}_t(a)+\sqrt{{3\log \ci}/{N_t(a)}}$, using \eqref{eq:deltahalf}, we have that 
\begin{equation}\label{eq:action1}
    a^*_t=i^*, \mbox{ and }\tilde{a}_t=i_1.
\end{equation}
This implies that under events $ \mathcal{E}_1(t)$, $ \mathcal{E}_2(t^*)$ and $ \mathcal{E}_3(t^*)$, we have
\begin{equation}
\begin{split}
     \hat\Delta^*_{t}&\geq\hat{\mu}_t(a^*_t)-\sqrt{\frac{3\log \ci}{ N_t(a^*_t)}}-\hat{\mu}_t(\tilde{a}_t)-\sqrt{\frac{3\log \ci}{  N_t(\tilde{a}_t)}}\\
    &\stackrel{(a)}{\geq} \mu_{i^*}-\mu_{i_1}-2\sqrt{\frac{3\log \ci}{ N_t(i^*)}}-2\sqrt{\frac{3\log \ci}{ N_t(i_1)}},\\
    &\stackrel{(b)}{>} (\mu_{i^*}-\mu_{i_1})(1-0.7-0.1),\\
    &\geq 0.2 \min_{i\neq i^*}\Delta(i),
\end{split}
\end{equation}
where $(a)$ follows from $\mathcal{E}_1(t)$, and $(b)$ follows from \eqref{eq:uppBound1} and \eqref{eq:uppBound2}. Thus, we have
\begin{equation}
\begin{split}
    \mathbb{P}(  \hat\Delta^*_{t}\leq 0.2 \min_{i\neq i^*}\Delta(i))&\leq \mathbb{P}(\bar{\mathcal{E}}_1(t)) + \mathbb{P}(\bar{\mathcal{E}}_2(t))+\mathbb{P}(\bar{\mathcal{E}}_3(t)),\\
    &\leq \frac{2K}{\ci^6}+\frac{4K^2}{\ci^5}+\frac{4K^2}{\ci^5}\leq\frac{10K^2}{\ci^5},
\end{split}
\end{equation}
where the last inequality follows from Hoeffding's inequality and Lemma \ref{lemma:LowerBoundsamples}.

\end{proof}

\begin{lemma}\label{lemma:whpBound}
For all $f(T)\leq t\leq T$, we have that
\begin{equation}\label{eq:whpBound}
    \mathbb{P}\bigg(N_t^s(i^*)> \max\bigg\{\frac{1200\log 2\ci}{(0.2\min_{i\neq i^*}\Delta(i))^2}, f(T)\bigg\}\bigg)\leq 14K^2/\ci^4.
\end{equation}
\end{lemma}
\begin{proof}
Consider the following events 
\begin{equation}
\mathcal{E}_1=\{\forall t\in \mathcal{T}: i_t=i^*\},
\end{equation}
and
\begin{equation}
    \mathcal{E}_2=\{\forall t\geq f(T): \hat\Delta_t^{*}\geq 0.2\min_{i\neq i^*}\Delta(i)\}.
\end{equation}
We will show that if $\mathcal{E}_1$ and $\mathcal{E}_2$ occurs, then  for all $f(T)\leq t\leq T$, we have
\begin{equation}\label{eq:tpProve1}
    N_t(i^*)\leq \max\bigg\{\frac{
    1200 \log 2\ci}{(0.2\min_{i\neq i^*}\Delta(i))^2},  f(T)\bigg\}. 
\end{equation}
For all $t\in \mathcal{T}$, let $\ell(t)=\max\{f(T)< t_1\leq t: i_{t_1}=i^* \mbox{ and } t_1\notin \mathcal{T}\}$ be the latest time instance before $t$ and after $f(T)$ where verification is performed for $i^*$. If $\ell(t)$ does not exists, then \eqref{eq:tpProve1} follows trivially. Then, under the events $\mathcal{E}_1$ and $\mathcal{E}_2$, for all $t\in\mathcal{T}$, we have that
\begin{equation}
\begin{split}
    N_t(i^*)&=N_{\ell(t)-1}(i^*)+1,\\
    &\stackrel{(a)}{\leq} \frac{
    1200 \log \ci}{\hat\Delta_{\ell(t)-1}^{*2}} +1,\\
    &\stackrel{(b)}{\leq} \frac{
    1200 \log 2T}{\hat\Delta_{\ell(t)-1}^{*2}},\\
    &\stackrel{(c)}\leq \frac{
    1200 \log 2\ci}{(0.2\min_{i\neq i^*}\Delta(i))^2}, 
\end{split}
\end{equation}
where $(a)$ follows from the fact that verification is performed for $i^*$ at round $\ell(t)$, $(b)$ follows from the facts that $1\leq 800\log 2$ and $\hat\Delta_t^{*2}\leq 1$, and $(c)$ follows from $\mathcal{E}_2$. 

Now, using Lemma \ref{lemma:LowerBoundsamples} and Lemma \ref{lemma:8}, we have
\begin{equation}
    \mathbb{P}(\bar{\mathcal{E}}_1) + \mathbb{P}(\bar{\mathcal{E}}_2)\leq \sum_{t=K}^{T}\frac{4K^2}{\ci^5}+\frac{10K^2}{\ci^5}\leq \frac{
    14K^2}{\ci^4}
\end{equation}
\end{proof}

\subsection{Finishing the Proof of Theorem \ref{thm:SUCB}}

We now finish the proof of Theorem \ref{thm:SUCB} as follows.
Let $t^*=\max\{t\in\mathcal{T}\}$. Then, the number of verifications is 
\begin{equation}
\begin{split}
    \sum_{i\in [K]}\mathbb{E}[N_T(i)]
    &=\sum_{i\in [K]}\mathbb{E}[N_{t^*}(i)],\\
     &\stackrel{(a)}{\leq } {\max\bigg\{f(T),\frac{1200\log 2\ci}{(0.2\min_{i\neq i^*}\Delta(i))^2}\bigg\}}+2K + \sum_{i\neq i^*}\frac{900\log \ci}{\Delta^2(i)}+\bigg(\frac{4K^2}{\ci^{5}}+\frac{14K^2}{T^{4}}\bigg)t^*,\\
     &\stackrel{}{\leq} {\max\bigg\{f(T),\frac{1200\log 2\ci}{(0.2\min_{i\neq i^*}\Delta(i))^2}\bigg\}}+2K + \sum_{i\neq i^*}\frac{900\log T}{\Delta^2(i)}+\frac{18K^2}{T^3},\\
     &\stackrel{}{\leq} {\max\bigg\{f(T),\frac{1200\log 2\ci}{(0.2\min_{i\neq i^*}\Delta(i))^2}\bigg\}}+2K + \sum_{i\neq i^*}\frac{900\log T}{\Delta^2(i)}+ \frac{18}{K},
\end{split}
\end{equation}
where $(a)$ follows from Lemma \ref{lemma:LowerBoundsamples} and Lemma \ref{lemma:whpBound}.
Using the definition of $f(T)$ in \eqref{eq:ftDef}, we have that \eqref{eq:numVerifSUCB} follows, namely
\begin{equation}
\begin{split}
        \sum_{i\in [K]}\mathbb{E}[N_T(i)]&\leq
        c_3\bigg(\sum_{i\neq i^*}{\log T}/{\Delta^2(i)}\bigg), 
\end{split}
\end{equation}

Consider the event
\begin{equation}
    \mathcal{E}=\{\forall t\in\mathcal{T} \mbox{ such that }t>K: i_t=i^*\}.
\end{equation}
Under event $\mathcal{E}$, for all $i\in [K]\setminus i^*$ and $t\leq T$, we have that 
\begin{equation}
    \sum_{t=1}^{T}\mathbb{P}(i_t=i)= N_t(i).
\end{equation}
The expected regret is
\begin{equation}\label{eq:ConditionalBound}
\begin{split}
\mathbb{E}[\sum_{t=1}^{T}\mathbf{1}(i_t=i)|\mathcal{E}]&=\mathbb{E}[N_T(i)|\mathcal{E}],\\
&=\mathbb{E}[N_{t^*}(i)|\mathcal{E}]\\
&\stackrel{(a)}{\leq}\frac{900\log T}{\Delta^2(i)}+\frac{4K^2 t^*}{T^{5}}+1\\
&\stackrel{(b)}{\leq}\frac{900\log T}{\Delta^2(i)}+\frac{4K^2}{T^4}+1\\
&\stackrel{(c)}{\leq}\frac{900\log T}{\Delta^2(i)}+\frac{4}{K^2}+1,
\end{split}
\end{equation}
where $(a)$ follows from Lemma \ref{lemma:LowerBoundsamples}, $(b)$ follows from the fact that $t^*\leq T$, and $(c)$ follows from the fact that $K\leq T$. 

Also, using Lemma \ref{lemma:LowerBoundsamples}, we have
\begin{equation}\label{eq:ProbBound}
    \mathbb{P}(\bar{\mathcal{E}})\leq \sum_{t=K}^T \frac{4K^2}{\ci^5} \leq \frac{4K^2}{T^4}
\end{equation}
Now, combining \eqref{eq:ConditionalBound} and \eqref{eq:ProbBound}, for all $i\in[K]\setminus i^*$, we have that 
\begin{equation}
    \mathbb{E}[N_T(i)]=\frac{900\log T}{\Delta^2(i)}+\frac{4}{K^2}+1 +\frac{4K^2}{T^3}.
\end{equation}
This implies that the regret of the algorithm is
\begin{equation}
\begin{split}
    R^{SUCB}(T)&=\sum_{i\neq i^*}\Delta(i)\mathbb{E}[N_T(i)],\\
    &=\sum_{i\neq i^*}\bigg( \frac{900\log T}{\Delta(i)}+\frac{4\Delta(i) }{K^2}+\frac{4K^2\Delta(i)}{T^3}+\Delta(i)\bigg). 
\end{split}
\end{equation}
Hence, we have that \eqref{eq:regretSUCB} follows.


\end{document}